\def\year{2015}
\documentclass[letterpaper]{article}
\usepackage{aaai}
\usepackage{times}
\usepackage{helvet}
\usepackage{courier}
\usepackage{url}
\usepackage{amsmath}
\usepackage{amsfonts}
\usepackage{amssymb}
\usepackage{graphicx}
\usepackage{algorithm}
\usepackage{algorithmic}

\newcommand{\SC}{\textrm{X}}
\newcommand{\AL}{\textrm{U}}
\newcommand{\R}{\mathbb{R}}
\newcommand{\Ls}{L}
\newcommand{\w}{w}
\newcommand{\M}{M}
\newcommand{\K}{k}

\newcommand{\gV}{V_{\text{greedy}}}
\newcommand{\gQ}{Q_{\text{greedy}}}
\newcommand{\Q}{Q}
\newcommand{\gmax}{\mbox{$\displaystyle\max^{g}$}}
\newcommand{\gdV}{V_{gd}}
\newcommand{\gdQ}{Q_{gd}}
\newcommand{\cdQ}{Q_{\widehat{gd}}}
\newcommand{\T}{\mathbb{T}}
\newcommand{\gdT}{\mathbb{T}_{gd}}
\newcommand{\gT}{\mathbb{T}_{\text{greedy}}}
\newcommand{\gsT}{\mathbb{T}_{\text{greedy}}^{\text{simple}}}
\newcommand{\cdT}{\mathbb{T}_{\widehat{\text{gd}}}}

\newcommand{\TBb}{\beta}

\newcommand{\ignore}[1]{}
\newcommand{\de}{\epsilon_{d}}
\newcommand{\gset}{\boldsymbol{G}}
\newcommand{\C}{\boldsymbol{c}}
\newcommand{\V}{V}

\newcommand{\p}{p}
\newcommand{\eps}{\epsilon}
\newcommand{\subeps}{\epsilon}

\newcommand{\B}{B}

\newtheorem{theorem}{Theorem}
\newtheorem{lemma}[theorem]{Lemma}
\newtheorem{observation}[theorem]{Observation}
\newtheorem{definition}[theorem]{Definition}
\newtheorem{prop}[theorem]{Proposition}
\newtheorem{proper}[theorem]{Property}
\newtheorem{assum}{Assumption}{\bfseries}{\itshape}
{\bfseries}{\itshape}
\newcommand{\proof}{\noindent{\bf Proof:\ }}
\newcommand{\qed}{\\$\square$}

\newcommand{\dd}[1]{\textcolor{green}{(DD: #1)}}

\begin{document}
	%
	\title{One-Shot Session Recommendation Systems with Combinatorial Items}
	\author{Yahel David\quad\quad Dotan Di Castro\quad\quad Zohar Karnin}

	\maketitle
	


%





\begin{abstract}
	In recent years, content recommendation systems in large websites (or \emph{content providers}) capture an increased focus. While the type of content varies, e.g.\ movies, articles, music, advertisements, etc., the high level problem remains the same. Based on knowledge obtained so far on the user, recommend the most desired content. In this paper we present a method to handle the well known user-cold-start problem in recommendation systems. In this scenario, a recommendation system encounters a new user and the objective is to present items as relevant as possible with the hope of keeping the user's session as long as possible. We formulate an optimization problem aimed to maximize the length of this initial session, as this is believed to be the key to have the user come back and perhaps register to the system. In particular, our model captures the fact that a single round with low quality recommendation is likely to terminate the session. In such a case, we do not proceed to the next round as the user leaves the system, possibly never to seen again. We denote this phenomenon a \emph{One-Shot Session}. Our optimization problem is formulated as an MDP where the action space is of a combinatorial nature as we recommend in each round, multiple items. This huge action space presents a computational challenge making the straightforward solution intractable. We analyze the structure of the MDP to prove monotone and submodular like properties that allow a computationally efficient solution via a method denoted by \emph{Greedy Value Iteration} (G-VI).

\end{abstract}

\section{Introduction}

In the user cold-start problem a new user is introduced to a recommendation system.
Here, the system often has little to no information about this new user and must provide reasonable recommendation nonetheless. A good recommendation system must on one hand provide quality (initially based on item popularity) recommendations to such users in order to keep them engaged, and on the other hand learn the new users' personal preferences as quickly as possible. The initial session of a user with a recommendation system is critical as in it, the user decides whether to terminate the session, and possibly never return, as opposed to registering to the site or becoming a regular visitor of the system. We refer to this phenomenon as that of a \emph{one-shot session}. This brings emphasis on the need to obtain guarantees not only for a long horizon but also for a very short one.

The one-shot session framework leads
to a highly natural objective: Maximize the session length, i.e.\ the number of items consumed by the user until terminating the session.
Indeed, the longer the user engages with the system the more likely she is to register and become a regular user. Our focus is on recommendation systems in which we present multiple items in each round. The user will either choose a single item and proceed to the next round, or choose to terminate the session. The property of having multiple items allows us to learn about the user's preferences based on the items chosen, versus those that were skipped.


A typical session length is quite short as it consists of a handful of rounds. This translates to us having very few data to learn from in order to personalize our recommendations. Due to the limited amount of information we are forced to restrict ourselves to a very simple model. For this reason we  take a similar approach to that in \cite{agrawal1989asymptotically,salomon2011deviations,MaillardM14} and assume that each user belongs to one of a fixed number of $\M$ user types (in the mentioned works these were called user clusters), such as man/woman, low/high income, or latent types based on previously observed sessions. The simplicity of the model translates into $\M$ being a small integer.
We assume that the model associated with each of the $\M$ user types is known\footnote{Learning the correct model for a user type can be done for example from data collected from different users whose identity is known. In either case this can be handled independently hence we do not deal with this issue}. That is, for any $\K$-tuple of items, the probability of each of the items to be chosen, and the probability of the session terminating given the user type is known. 
We emphasize the fact that a complete recommendation system will start with the simple model with $\M$ being a small constant, and for users that are `hooked', i.e.\ remain for a long period / register, we may move to a more complex model where for example a user is represented by a high dimensional vector. We do not discuss the latter more complex system, aimed for users with a long history, as it is outside the scope of our paper.

The problem we face can be formulated as a Markov Decision Problem (MPD; \cite{bertsekas1996neuro,sutton1998introduction}). In each round the state is a distribution over $[\M]$ reflecting our knowledge about the user. We choose an action consisting of $k$ different items from the item set $L$. The user either terminates the session, leading to the end of the game or chooses an item, moving us to a different state as we gained some knowledge as to her identity. Notice that any available \emph{context}, e.g.\ time of day, gender, or basic information available to us can be used in order to set the initial state. The formulated MDP can be solved in order to obtain the optimal strategy; the computational cost scales as the size of the action space and the state space. Since $\M$ is restricted to be small, the size of the state space does not present a real challenge. However, the action space has a size of $|L|^k$ which is typically huge. The number of available items can be in the hundreds if not thousands and a system presenting even a handful of items will have for the very least billions of possible actions. For this reason we seek a solution that scales relatively to $k|L|$ rather than $|L|^k$.

To this end we require an additional mild assumption, that can be viewed as a quantitive extension of the  \emph{irrelevant alternatives axiom} (see Section~\ref{sec:assump}). 
With this assumption we are able to provide a solution (Section~\ref{sec:theory}) based on a greedy approach that scales as $\K|\Ls|$ and has a constant competitive ratio with the computationally unbounded counterpart.
The main component of the proof is an analysis showing that the sub-modularity and monotonicity of the \emph{immediate reward in a round} translates into monotone and sub-modular-like properties of the so called \emph{$Q$-function} in a modified value iteration procedure we denote by \emph{Greedy Value Iterations} (G-VI). Given these properties we are able to show, via an easy adaptation of the tools provided in \cite{nemhauser1978analysis}  for dealing with submodular monotone functions, that the greedy approach emits  a constant approximation guarantee. We emphasize that in general, a monotone submodular reward function does not in any way translate into a monotone submodular $Q$ function, and we exploit specific properties of our model in order to prove our results; to demonstrate this we show in Appendix~\ref{app:Q_example} an example for a monotone submodular immediate reward function with a corresponding $Q$ function that is neither monotone nor submodular. We complement the theoretical guarantees of our solution in Section~\ref{add_exp} with experimental results on synthetic data showing that in practice, our algorithm has performance almost identical to that of the computationally unbounded algorithm. 

\section{Related Work}
Many previous papers provide adaptive algorithms for managing a recommendation system, yet to the best of our knowledge, non of them deal with one-shot sessions. The tools used include Multi-armed Bandits \cite{radlinski2008learning}, Multi-armed bandits with submodularity, \cite{yue2011linear},  MDPs \cite{shani2005mdp}, and more. A common property shared by these results is the assumption of an infinite horizon. Specifically, a poor recommendation given in one round cannot cause the termination of the session, as in one-shot sessions, but only result in a small reward in the same single round.
 This crucial difference in the `cost' of a single bad round in the setups of these papers versus ours is very likely to cause these methods to fail in our setup. 
A paper that partially avoids this drawback is by \cite{deshpande2012linear}, where other than a guarantee for an infinite horizon the authors provide a multiplicative approximation to the optimal strategy at all times. A notable difference between our setup is the fact that the recommendations there consist of a single item rather than multiple items as required here. This, along with the somewhat vague connection to our one-shot session setup exclude their methods from being a possible solution to our problem.

Our work can be casted as a Partially Observable MDP (POMDP; \cite{kaelbling1998planning}), where the agent only has partial (sometimes stochastic) knowledge over the current state. Our problem stated as a POMDP instance admits  $\M+1$ states, one for each user type and an additional state reflecting the session end. The benefit of such an approach is the ability to significantly reduce the size of the state space, from $\exp(\M)$ potentially down to $\M+1$. Nevertheless, we did not chose this approach as the gain is rather insignificant due to $M$ being a small constant, while the inherent complication to the analysis and algorithm make it difficult to deal with the large action space, forming the main challenge in our setting.  Recently, \cite{satsangi2015exploiting} presented a result dealing with a combinatorial action space in a POMDP framework, when designing a dynamic sensor selection algorithm. They analyze a specific reward function that is affected only by the level of uncertainty of the current state, thereby pushing towards a variant of pure exploration. The specific properties of their reward function and MDP translate into a monotone and submodular $Q$-function. These properties are not present in our setup, in particular due to the fact that a session may terminate, hence the methods cannot be applied. Furthermore, our greedy VI variant is slightly more complex than the counterpart in \cite{satsangi2015exploiting} as it is tailored to ensure the (approximate) monotonicity of $Q$; this is an issue that was not encountered in the problem setup of \cite{satsangi2015exploiting}.


Another area which is related to our work is that of ``Combinatorial Multi Armed Bandits" setup (C-MAB; see \cite{chen2013combinatorial} and references within). Here, similarly to our setup, in each round the set of actions available to us can be described as subsets of a set of options (denoted by arms in the C-MAB literature). These methods cannot directly be applied to our setting due to the infinite horizon property mentioned above. Furthermore, the methods given there that help deal with the combinatorial nature of the problem cannot be applied in our setting since the majority of our efforts lie in characterizing properties of the $Q$-function; an object that has no meaning in MAB settings but only in MDPs.
\section{Problem Formulation}
In this section we provide the formal definition of our problem. We first provide the definition of a Markov Decision Process (MDP). We continue to describe our setup and its different notations, and then formulate it as an MDP. 

\subsection{Markov Decision Processes}
An MDP is defined by a tuple $\left\langle \SC,\AL,P,R \right \rangle$ where $\SC$ is a state space, $\AL$ is a set of actions, $P$ is a mapping from state-action pairs to a probability distribution over the next-states, and $R$ is a mapping from the state-action-next-state to the reward. The MDP defines a process of rounds. In each round $t$ we are at a state $\C \in \SC$ and must choose an action from $\AL$. According to our action, the following state and the reward $r_t$ are determined according to $P,R$. The objective of an MDP is to maximize the cumulative sum of rewards with a future discount of $\gamma<1$, i.e.\ $\sum_{t=0}^\infty \gamma^t R(\C^t, w^t)$, where $w^t$ is the action taken at time $t$, $\C^t$ is the state at time $t$, and $R(\C^t,w^t)$ is the expected reward given the action-state pair. For this objective we seek a policy $\pi$ mapping each state to an action. The objective of planing in an MDP is to find a policy $\pi$ maximizing the value function
\begin{equation*}
\V^{\pi}(\C) \triangleq E \left[ \left. \sum_{t=0}^{\infty}\gamma^{t}R \left(\C^{t},\pi\left(\C^{t} \right)\right) \right| \C^{0}=\C,\pi \right],
\end{equation*}
where the value of $V^{\pi}(\C)$ is the long-term accumulated reward obtained by following the policy $\pi$, starting in state $\C$. 
We denote the optimal value function by $\V^{*}(\C)=\sup_{\pi}\V^{\pi}(\C)$. A policy $\pi^{*}$ is optimal if its corresponding value function is $\V^{*}$ (see \cite{bertsekas1996neuro} for details).

The \emph{Bellman's operator} (or \emph{DP operator}) maps a function $\V:\SC \to \R^+$ (where $\R^+$ is the set of non-negative reals) to another function $(\T\V):\SC \to \R^+$ and is defined as follows.
\begin{equation}
\label{eq:DP}
(\T\V)(\C)=\max_{\w\in U} \sum_{\C'} \left( R(\C,\w,\C')+\gamma\V(\C') \right) P(\C'|\C,w)  ,
\end{equation}
where  $\C$ and $\C'$ denote the current  and next state, respectively.

Under mild conditions, the equation $\V(\C) = (\T\V)(\C)$ is known to have a unique solution which is the fixed point of the equation and equals to $\V^{*}$. A known method for finding $V^{*}$ is the Value Iteration (VI; \cite{bertsekas1996neuro,sutton1998introduction}) algorithm which is defined by applying The DP operator \eqref{eq:DP} repeatedly on an initial function $\V^0$ (e.g.\ the constant function mapping all states to zero). More precisely, applying \eqref{eq:DP} $t$ times on $V^0$ yields $V^t \triangleq \T^t\V^0$ and the VI method consists of estimating $\lim_{t \rightarrow \infty} V^t$.
The VI algorithm is known to converge to $\V^{*}(\C)$. However, computational difficulties arise for large state and action spaces. 

\subsection{Notations}
Let us first formally define the rounds of the user-system interaction and our objective. When a new user arrives to the system (e.g., content provider) we begin a \emph{session}. At each round, we present the user a subset of up to $k$ items from the set of available items $\Ls$. The user either terminates the session, in which case the session ends, or chooses a single item from the set, in which case we continue to the next round. The reward is either $r=1$ if the user chose an item\footnote{It is an easy task to extend our results to a setting where different items incur different rewards. For simplicity however we keep it simple and assume equality between items, in terms of rewards.} or $r=0$ otherwise. Following a common framework for MDPs, our objective is to maximize the sum of rewards with future rewards discounted by a factor of $\gamma$. That is, by denoting $r_t$ the reward of round $t$ and $T$ the random variable (or random time) describing the total number of rounds, we aim to maximize
\begin{equation}
\label{eq:discounted_objective}
\mathbb{E} \left[ \sum_{t=0}^{T-1} \gamma^t r_t \right] .
\end{equation}
The reason for considering $\gamma<1$ is the fact that the difference between a session of say length 10 and length 5 is not the same as that of length 6 and 1. Indeed in the user cold-start problem one can think of a model where every additional item observed by the user increases the probability of her registering, yet this function is not linear but rather monotone increasing and concave.

We continue to describe the modeling of users. Recall that users are assumed to characterized by one of the members of the set $[\M]$. Our input contains for every set $w \subseteq \Ls$ of items, every user type $m \in [\M]$, and any item $\ell \in w$ the probability $p(\ell | m, w)$ of the user of type $m$ choosing item $\ell$ when presented the set $w$. 
In the session dynamics described above we maintain at all times a belief regarding the user type, denoted by\footnote{Eventually we consider a discretization of the simplex, but for clarity we discuss this issue only at a  later stage.} $\C \in \Delta_\M$, with $\Delta_\M$ being the set of distributions over $[\M]$. Notice that given the distribution $\C$ we may compute for every set $w$ and item $\ell \in w$ the probability of the user choosing item $\ell$. We denote this probability by $$p(\ell | \C, w) = \sum_{m \in [\M]} \C(m) \cdot p(\ell | m, w) $$
Assume now that at round $t$, our belief state is $\C_t = \C$, we presented the user a set of items $w$, and the user chose item $\ell$. The following observation provides the posterior probability $\C_{t+1}$ also denoted by $\C'_{\ell,w,\C}$. The proof is based on the Bayes rule; as it is quite simple we defer it to Appendix~\ref{lem:posterior:supp} in the supplementary material.

\begin{observation}
\label{lem:posterior}
The vector $\C'_{\ell, w, \C}$ is the posterior type-probability for a prior $\C$, action $w$ and a chosen item $\ell$. This probability is obtained by
	\begin{equation}\label{eq:posterior_b}
	\C'_{\ell,w,\C}(m')=\frac{\p(\ell | m',\w)\C(m')}
	{\p(\ell | \C,\w)}.
	\end{equation}
\end{observation}

\subsection{Formulating the Problem as an MDP}

We formulate our problem as an MDP as follows. The state space $X$ is defined as $\Delta_\M \cup \{\C_\emptyset\}$ where $\C_{\emptyset}$ denotes the termination state. The action space $U$ consists of all subsets $w \subseteq \Ls$ of cardinality $|w| \leq k$. The reward function $R$ depends only on the target state and is defined as 1 for any $\C \in \Delta_\M$ and zero for $\C_\emptyset$. As a result of Observation \ref{lem:posterior}, we are able to define the transition function $P$:
\begin{equation*}
\begin{aligned}
&P(\C'|\C,\w)=
&
\begin{cases}
\displaystyle\sum_{\ell \in \Ls(\C',\C,\w)} \displaystyle p(\ell | \C,\w) &\Ls(\C',\C,\w)\neq\emptyset\\
0 &\Ls(\C',\C,\w)=\emptyset
\end{cases}
\end{aligned},
\end{equation*}
where the set $\Ls(\C',\C,\w)$ is defined as
$$ \left\{ \ell \ \middle| \  \forall m', \ \C'(m') = \frac{p(\ell | m', \w) \C(m')}{\sum_{m \in M} p(\ell|m, \w) \C(m))} \right\} $$
that is the set containing $\ell \in \Ls$ such that \eqref{eq:posterior_b} is satisfied. The final missing definition to the transition function is the probability to move to the \emph{termination} state, denoted by $\C_\emptyset$, defining the session end. For it, $P(\C_\emptyset | \C, w) = 1- \sum_{\ell \in w} p(\ell | \C, w)$.

\section{User Modeling Assumptions} \label{sec:assump}
In order to obtain our theoretical guarantees we use assumptions regarding the user behavior. Specifically, we assume a certain structure in the function mapping a item set $\w$ and a item $\ell \in \w$ to the probability that a user of type $m$ (any $m$) will choose the item $\ell$ when presented with the item set $\w$. To assess the validity of the below assumption consider an example standard model\footnote{An example for where this modeling is implicitly made is in the setting of a \emph{Multinomial Logistic Regression}.} where each item $\ell$ in $\w$ (and the empty item) has a positive value $\mu_\ell$ for the user and the chosen item is drawn with probability proportional to $\mu_\ell$. We note that the below assumptions hold for this model.

The first assumption essentially states that at all states there is a constant, bounded away from zero, probability to reach the termination state. In our setup this translates into an assumption that even given knowledge of the user type, the probability of the user ending the session remains non-zero. Needless to say this is a highly practical assumption.

\begin{assum}\label{assum:sumP_B}
For a constant $\B>1$, any set of content items $\w\in\Ls^{i}$ where $i\leq k$, any types vector $\C\in \Delta_\M$ and a content item $\ell \in\Ls$, it holds that
\begin{equation*}
\sum_{\ell\in\Ls}\p(\ell| \C,\w)\leq \frac{1}{\B}.
\end{equation*}
\end{assum}

In what follows, our approximation guarantee will depend on $B$, that is on how much the best-case-scenario probability of ending a session is bounded away from zero. The second assumption assert independence between the probabilities of choosing different content items.

\begin{assum}\label{assum:w_w'}
For every $m\in\M$, a set of content items $\w$ and a content item $\ell'\not\in\w$ it holds that
\begin{equation}\label{eq:assum:w_w'}
p(\ell|m,w)=p(\ell|m,w\cup \ell')+p(\ell'|m,w\cup \ell')p(\ell|m,w).
\end{equation}
\end{assum}

The above assumption is related to the \emph{independence of irrelevant alternatives axiom} (IIA)~\cite{saari2001decisions} of decision theory, stating that ``If $A$ is preferred to $B$ out of the choice set $\{A,B\}$, introducing a third option $X$, expanding the choice set to $\{A,B,X\}$, must not make $B$ preferable to $A$''. Our assumption is simply  a quantitive version of the above.

\section{Approximation Of the Value Function}\label{sec:theory}
In this section we develop a computationally efficient approximation of the value function for the setup described above. We begin with dealing with the action space, and later we also take into consideration the continuity of the state space.

\subsection{Addressing the Largeness of the Action Space by Sub-modularity} \label{sec:action_space}
In this section we provide a greedy approach dealing with the large action space, leading to a running time scaling as $O(\K|L|+|\SC|)$. For clarity we ignore the fact that $\SC$ is infinite and defer its discretization to the Section~\ref{sec:discrete}. The outline of the section is as follows: We first mention that the immediate reward function, when viewed as a function of the action, is monotone and submodular. Next, we define a modified value-iteration procedure we denote by \emph{greedy value iteration} (G-VI), resulting in a sequence of approximate value function $V^t$ and $Q$-functions $Q^t$, obtained in the iterations of the procedure. We show that these $Q^t$ functions are approximately monotone and approximately submodular and that for functions with these approximate monotone-submodular properties, the greedy approach provides a constant approximation for maximization; we are not aware of papers using the exact same definitions for approximate monotonicity and submodularity yet we do not consider this contribution as major since the proofs regarding the greedy approach are straightforward given existing literature. Finally, we tie the results together and obtain an approximation of the true $Q$ function, as required. 

Since it is mainly technical and due to space limitations, we defer the proof that the reward function is monotone and submodular to Appendix~\ref{app:prop2}.
We now turn to describe the process G-VI.
We start by defining our approximate maximum operator
\begin{definition}
Let $L$ be a set, $f: L \to \R$, and let $0 \leq k \leq |L|$ be an integer. We denote by $L^k$ the set of subsets of $L$ of size $k$.
The operators $\gmax, \arg \gmax$ (the superscript ``g" for greedy) are defined as follows 
$$\gmax_{w\in\Ls^{0}} f(w)= f(\emptyset), \ \ \ \ \arg \gmax_{w\in\Ls^{0}} f(w)=\emptyset ,$$
$$\gmax_{w\in\Ls^{k+1}} f(w)= \max_{\ell \in L} f(\arg \gmax_{w'\in\Ls^{k}} f(w') \cup \{\ell\}), $$
$$\begin{aligned}
&\arg \gmax_{w\in\Ls^{k+1}} f(w)=
\arg \gmax_{w\in\Ls^{k}} f(w) 
\cup \\&\arg \max_{\ell \in L} f\left(\arg \gmax_{w\in\Ls^{k}} f(w) \cup \{\ell\} \right).\end{aligned} $$
\end{definition}
Informally, the $\gmax$ operator maximizes the value of a function $f$ over subsets of restricted size by greedily adding elements to a subset in a way that maximizes $f$. 
For a value function $\V$ we define the $\Q$ function as 
\begin{equation} \label{eq:def_Q_func}
\Q_\V(\w',\C)=\sum_{\ell\in\Ls} \p(\ell| \C,\w')(1+\gamma \V(\C'_{\ell,\w',\C}))
\end{equation}
When it is clear from context which $\V$ is referred to, we omit the subscript of it. Recall that the standard DP operator is defined as 
$\left(\T \V\right)(\C)=\max_{w\in L^k} \Q(\w,\C)$.
Using our greedy-based approximate max we define two greedy-based approximate DP operator. The first is denoted as the simple-greedy approach where 
 \begin{equation}\label{eq:simple:greedy}
(\gsT \V)(\C)= \gmax_{\w}  \Q(\w,\C)
\end{equation}
As it turns out, the simple-greedy approach does not necessarily converge to a quality value function. In particular, the $Q$ function obtained by it does not emit necessary monotone-submodular-like qualities that we require for our analysis. We hence define the second DP operator we call the greedy operator.
\begin{definition}\label{def:greedyV}
For a function $V:X \to \R^+$ we define 
\begin{equation}\label{eq:def:greedyV}
\left(\gT \V\right)(\C)= \max_{w\in \gset}  \Q_\V (\w,\C)
\end{equation}
where the set $\gset$ is defined in the following statement,
\begin{equation*}
\gset=\{\w|\exists \C\in\SC\, s.t\, \w=\arg\gmax_{\w'\in\Ls^{k}}\Q(\w',\C)\},
\end{equation*}
\end{definition}
In words, we take advantage of the fact that the number of states is small (as opposed to the number of actions) and use the $\gmax$ operator not to associate actions with states but rather to reduce the number of actions to be at most the same as the number of states. We then choose the actual $\arg \max$ for each state, from the small subset of actions. Notice that the compositional complexity of the $\gT$ operator is $O(\K|L|+|\SC|)$, as opposed to $O(\K|L|)$ as the $\gsT$ operator.
In Appendix \ref{add_exp} \ignore{\dd{Yahel, please refer this section to the right place in the appendix}}, we explore whether there is a need for the further complication involved with using $\gT$ rather than $\gsT$, or whether its use is needed only for the analysis. We show that in simulations, the system using the $\gT$ operator significantly outperforms that using the simpler $\gsT$ operator.

Recall that the value iteration (VI) procedure consists of starting with an initial value function, commonly the zero function, then performing the $\T$ operator on $\V$ multiple times until convergence. Our G-VI process is essentially the same, but with the $\gT$ operator. Specifically, we initialize $\V$ to be the zero function and analyze the properties of $\gT^t \V$ for $t >0$. In our analysis we manage to tie the value of $\gT^t \V$ computed w.r.t.\ a decay value $\gamma$ (Equation~\eqref{eq:def_Q_func}), to the value of $\T^t \V$, the true VI procedure, computed w.r.t.\ a decay value of $\gamma'$ with $\gamma' \approx 0.63 \gamma$. To dispaly our result we denote by $\T^t_{\gamma'} \V$ the iterated DP operator done on $\V$ w.r.t.\ decay value $\gamma'$. The proof is given in Appendix~\ref{thm:g:supp} .

\begin{theorem}\label{thm:g}
Let $\gamma > 0$. Under Assumptions \ref{assum:sumP_B} and \ref{assum:w_w'}, for $B\geq 2$, zero initiation of the value function (namely, $\V=0$) and for any $t\geq1$, it is obtained that
\begin{equation}\label{thm:greedy:first}
\left(\gT^{t}\V\right)(\C)\leq (\T^{t}\V)(\C)
\end{equation}
\begin{equation}\label{eq:ind:assum}
\TBb\left((\T^{t}_{\TBb\gamma} \V)(\C) -\Omega_{t,\C}\right)\leq(\gT^{t}\V)(\C)
\end{equation}
with $\TBb = 1-1/e \approx 0.63$, 
$$\Omega_{t,\C}\triangleq\sum_{i=0}^{t-1}\left(\TBb\gamma\rho(\C)\right)^{i}(k-1)\overline{\theta}(\C),$$
$$\rho(\C)\triangleq\max_{\w\in\Ls^{\K}}\sum_{m\in\M}\C(m)\sum_{\ell\in\Ls}P(\ell|m,w),$$ and
\begin{equation}\label{THM:Theta}
\begin{aligned}
&\overline{\theta}(\C)\triangleq
\max_{\ell'\in\Ls,\w\in\Ls^{\K}}\\&\sum_{m\in\M}\C(m)P(\ell'|m,w\cup \ell')\sum_{\ell\in\Ls}P(\ell|m,w)\frac{\gamma}{B-\gamma}
\end{aligned}
\end{equation}
\end{theorem}

To better understand the meaning of the above expression we estimate the value of $\Omega_{t,\C}$ for the initial state $\C$ in reasonable settings. Specifically, we would like estimate 
\begin{equation*}
\lambda \triangleq\frac{\Omega_{t,\C}}{(\T^{t}_{\TBb\gamma} \V)(\C)}.	
\end{equation*}
In cases where $\lambda$ is a small constant we get a constant multiplicative approximation of the value function obtained via the optimal, computationally inefficient maximization.

In the supplementary material (Lemma~\ref{lem:bound:sum2}) we provide the bound
$$\lambda = \frac{\Omega_{t,\C} }{  (\T^{t}_{\TBb\gamma}\V)(\C)} \leq \frac{(k-1) \bar{\theta}(\C)}{\rho(\C)}$$ 
The proof is purely technical. Notice that $\rho(\C)$ is in fact the probability of the user, given the state $\C$ and us choosing the best possible action, choosing a link rather than terminating the session. Assuming a large number of content items (compared to $k$) it is most likely that  for every type $m\in\M$ there are much more than $\K$ favorable items. This informally means that either the probability of choosing any item $\ell$ among a set  $w$ is roughly $\rho(\C)/k$ or $w$ is a poor choice of links and the probability of ending the session when presenting $w$ is significantly lower than $\rho(\C)$. It is thus reasonable to assume that 
$$ (k-1) \overline{\theta}(\C)\lesssim\rho(\C)^2\frac{\gamma}{(B-\gamma)}.$$ 
Hence
$$\lambda  \lesssim \frac{\rho(\C) \gamma}{(B-\gamma)} \leq \frac{\gamma}{B(B-\gamma)}$$ 
For example, for $B=2$ and $\gamma=0.75$. Then we have $\lambda \leq 0.3$, hence $0.44(\T^{t}_{\TBb\gamma} \V)(\C) \leq(\gT^{t}\V)(\C)$, meaning we get a multiplicative $0.44$ approximation compared to the optimal operator with $\gamma' \approx 0.47 $.

\subsection{Addressing Both The Continuity of State Space and The Largeness of the Action Space} \label{sec:discrete}
Recall that the state space  of our model is continuous. As our approach requires scanning the state space we present here an analysis of our approach taken over a discretized state space. That is, rather than working over $\Delta_{\M}$ (the entire $\M$ dimensional simplex) our finite state space $\SC$ is taken to be an $\eps$-net, w.r.t.\ the $L_1$-norm, over $\Delta_{\M}$.


As before, the value iteration we suggest takes the greedy approach where the only difference is in the definition of the $Q$-function.
\begin{definition}
The $\gdQ$-function, based on a function $\gdV^{t-1}(\C)$ mapping a state to a value is defined as follows:
\begin{equation}
\gdQ^{t}(\w,\C)= \sum_{\ell\in\Ls}\p(\ell | \C,w)(1+\gamma \gdV^{t-1}(\widehat{\C'_{\ell,w,\C}})).
\end{equation}
where $\widehat{\C'_{\ell,w,\C}}$ is defined as the closest point in $\SC$ to $\C'_{\ell,w,\C}$.
	
\end{definition}			

Analogically to before, we define the $\gdT$ operator over a value function $\gdV$ as 
\begin{equation}\label{eq:def:greedyV:disc}
\gdV^t(\C) =  (\gdT \gdV^{t-1})(\C) = \max_{\w \in \gset} \gdQ(\w, \C)
\end{equation}

with $\gset$ being defined w.r.t.\ the finite state set $\SC$. In Appendix~\ref{thm:g_d:supp} we prove the following theorem, giving the analysis of the above value iteration procedure.

\begin{theorem}\label{thm:g_d}
Under Assumptions \ref{assum:sumP_B} and \ref{assum:w_w'}, for $B\geq 2$, zero initiation of the value function (namely, $\V=0$), a state space formed via an $L1$ $\epsilon$-net, and for any $t\geq1$ we have
$$\left(\gdT^{t}\V\right)(\C)\leq (\T^{t}\V)(\C)+O(\eps)$$
$$\TBb\left((\T^{t}_{\TBb\gamma} \V)(\C) -\Omega_{t,\C}(1+O(\eps)) \right) \leq(\gdT^{t}\V)(\C)$$
where $\TBb, \Omega_{t,\C}$ are the same as in Theorem~\ref{thm:g}.
\end{theorem}

For sufficiently small $\eps= \Omega(1)$, the result is essentially the same as that in Section~\ref{sec:action_space}.

\section{Experiments \protect\footnote{Additional experiments are provided in Section \ref{sup:add:exp} of the supplementary material.}}\label{add_exp}
In this section we investigate numerically the algorithms suggested in Section \ref{sec:theory}. We examine four types of CP policies:\\
\noindent \textbf{1. Random policy}, where the CP provides a (uniformly) random set of content items at each round. \\
\noindent \textbf{2. Regular DP operator policy}, namely $\T$ as in  \eqref{eq:DP}, in which the maximum is computed exactly. The computational complexity of each iteration of the VI with the original DP operator is of order of $O(|\SC||L|^{K})$. \\
\noindent \textbf{3. Greedy Operator policy}, namely following the $\gdT$ operator as in \eqref{eq:def:greedyV:disc}. In this case the computational complexity of each iteration of the G-VI is of order of $O(|\SC||L| K+|\SC|^{2})$. \\
\noindent \textbf{4. Simple Greedy CP}, namely following the $\gsT$ operator as in \eqref{eq:simple:greedy}. No theoretical guarantees are provided for this CP, but since its computational complexity of each iteration of the VI is of order of $O(|\SC||L| K)$ and its similarity to the \emph{greedy} CP, we are interested in its performances.

We conducted our experiments on synthetic data. The users' policy implemented the following model relating the scores to the users' choice,
\begin{equation*}
P(\ell|m,\w)=\frac{\ell_{m}}{\sum_{\ell'\in\w} \ell'_{m}+p_{m}}\ ,
\end{equation*}
where $\ell_{m}$ is a \emph{score} expressing the subjective value of item $\ell$ for users of type $m$ and where $p_{m}$ expresses the tendency of user of a type $m$ to terminate the session. It is easy to verify that for $p_{m}$ large enough compared to the scores, Assumption \ref{assum:sumP_B} holds, and that Assumption \ref{assum:w_w'} holds for any value assigned to $\ell_m$  and $p_{m}$. 
%

For the experiments, we considered the case of $\M=4$, $|\Ls|=13$, $\K=3$ and $\gamma=1$. The scores were chosen as follows: For all types, the termination score was $p_m=0.5$. Four items were chosen i.i.d.\ uniformly at random from the interval $[0, 0.6]$. The remaining $8$ items where chosen such that for each user type, $2$ items are uniformly distributed in $[0.5,1]$ (strongly related to this type), while the other $6$ are drawn uniformly from $[0,0.5]$. 
We repeated the experiment $500$ times, where for each repetition a different set of scores was generated and $100,000$ sessions were generated (a total of $50M$ sessions).

In Figure \ref{figure:experiment} we present the average session length under the \emph{optimal, greedy} and \emph{simple greedy} CPs for different numbers of iterations executed for computing the Value function. The average length that was achieved by the \emph{random} CP is $1.3741$, much lower than that of the other methods. The standard deviation is smaller that $10^{-3}$ in all of our measures. As shown in Figure \ref{figure:experiment}, the extra comparison step in the \emph{greedy} CP compared to the \emph{simple greedy} CP substantially improves the performance. 

\begin{center}
	\begin{figure}[ht]
		\centering
		\includegraphics[width=0.45\textwidth]{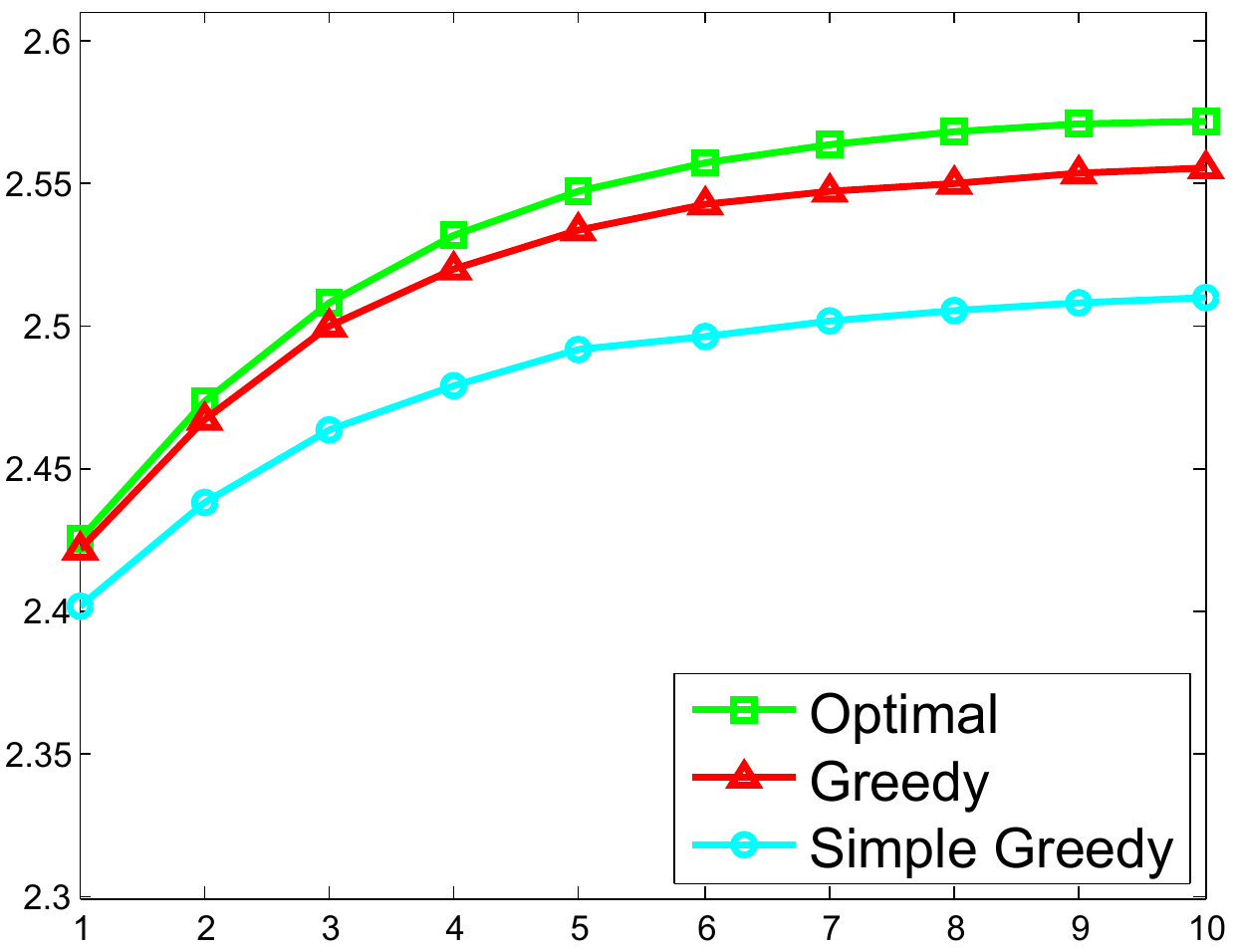}
		\caption{Average session length under the \emph{optimal, greedy} and \emph{simple greedy} ($y$-axis) CPs vs. number of iterations of the related VI computation ($x$-axis). The average length of the \emph{random} CP is $1.3741$ (not shown).} 
		\label{figure:experiment}
	\end{figure}
\end{center}




\section{Discussion and Conclusions}\label{Discussion:Exp}
In this work we developed a new framework for analyzing recommendation systems using the MDP framework. The main contribution is two-fold. First, we provide a model for the user-cold start problem with one-shot sessions, where a single round with low quality recommendations may end the session entirely. We formulate a problem where the objective is to maximize the session length
\footnote{
Another problem, which is somehow related to the cold-start problem, is the problem of devices that are shared between several users \cite{white2014devices,file2013computer}. In this scenario, several people share the same device while the content provider is aware only of the identity of the device and not of the identity of the user. This phenomenon typically occurs with devices in the same household, shared by the members of the family. The methods developed in this work can be easily adapted to solve this problem as well.
}.
%
Second, we suggest a greedy algorithm overcoming the computational hardship involved with the combinatorial action space present in recommendation system that recommend several item at a time. The effectiveness of our theoretical results is  demonstrated with experiments on synthetic data, where we see that our method performs practically as well as the computationally unbounded one.

As future work we plan to generalize our techniques for dealing with the combinatorial action space to setups other than the user-cold start problem, and aim to characterize the conditions in which the $Q$ function is (approximately) monotone and submodular. In particular we will consider an extension to POMDPs as well that may deal with similar settings in which $\M$ can take larger values.

\fontsize{9.5pt}{10.5pt}
\selectfont
\bibliography{bibliography}
\bibliographystyle{aaai}
\newpage

\onecolumn
\appendix

\section{Missing Proofs}
\subsection{Proof of Lemma \ref{lem:posterior}}\label{lem:posterior:supp}

Here we use $\C'$ as a short for $\C'_{\ell,w,C}$.
By Bayes' theorem, for any $\C$, $\w\in\Ls^{\K}$ and $\ell\in\Ls$, it follows that
\begin{equation}
\begin{aligned}
\C'(m')&=P(\C(m')=1|\ell,\w,\C)
\\&=
\frac{P(\C(m')=1,\ell|\w,\C)}
{P(\ell|\w,\C)}
\\&=\frac{P(\ell|\C(m')=1,\w,\C)P(\C(m')=1|\w,\C)}
{P(\ell|\w,\C)}
\\&=\frac{\p(\ell|m',\w)\C(m')}{\sum_{m\in\M}\p(\ell|m,\w)\C(m)}
\end{aligned}\ ,
\end{equation}
where $P(\C(m')=1)$ stands for the probability that the user type is $m'$.
So, the result is obtained.
\qed

\section{Additional Propositions and Lemmas}
In the following propositions and lemmas we derive some results related to greedy maximization of submodular functions. These results are used for the proofs of Theorems \ref{thm:g} and \ref{thm:g_d}.

\subsection{Model Properties}  \label{app:prop2}
In the following proposition we show the monotonicity and submodularity properties of the chosen model.
\begin{prop}\label{assum1_G}
	Under Assumption \ref{assum:w_w'}, for any two sets of content items $\w_{b}\supset\w_{a}$ and a content item $\ell'\not\in w_{b}$, it holds that (\emph{monotonicity})
	\begin{equation}\label{eq_assum1_G2}
	\left(\sum_{\ell\in\Ls}\p(\ell|m,\{w_{b}\cup \ell'\})-\sum_{\ell\in\Ls}\p(\ell|m,w_{b})\right)\geq0,
	\end{equation}
	and (\emph{submodularity})
	\begin{equation}\label{eq_assum1_G}
	\begin{aligned}
	&\left(\sum_{\ell\in\Ls}\p(\ell|m,\{w_{a}\cup \ell'\})-\sum_{\ell\in\Ls}\p(\ell|m,w_{a})\right)\geq
	&\left(\sum_{\ell\in\Ls}\p(\ell|m,\{w_{b}\cup \ell'\})-\sum_{\ell\in\Ls}\p(\ell|m,w_{b})\right)
	\end{aligned},
	\end{equation}
	for any type $m\in\M$.
\end{prop}

\proof
By Equation \eqref{eq:assum:w_w'} it follows that
\begin{equation}\label{prop:assum:eq:1}
\begin{aligned}
&\left(\sum_{\ell\in\Ls}\p(\ell|m,\{w_{b}\cup \ell'\})-\sum_{\ell\in\Ls}\p(\ell|m,w_{b})\right)=\\
&\p(\ell'|m,\{w_{b}\cup \ell'\})\left(1-\sum_{\ell\in\Ls}\p(\ell|m,w_{b})\right)\geq 0
\end{aligned}.
\end{equation}
So, Equation \eqref{eq_assum1_G2} is obtained.

For proving Equation \eqref{eq_assum1_G}, we note that
\begin{equation}\label{prop:assum:eq:2}
\begin{aligned}
&\left(\sum_{\ell\in\Ls}\p(\ell|m,\{w_{a}\cup \ell'\})-\sum_{\ell\in\Ls}\p(\ell|m,w_{a})\right)=\\
&\p(\ell'|m,\{w_{a}\cup \ell'\})\left(1-\sum_{\ell\in\Ls}\p(\ell|m,w_{a})\right)\geq 0
\end{aligned}.
\end{equation}
Then, since by Equation \eqref{eq:assum:w_w'} we have that
\begin{equation*}
\p(\ell'|m,\{w_{a}\cup \ell'\})\geq \p(\ell'|m,\{w_{b}\cup \ell'\}),
\end{equation*}
and that
\begin{equation*}
\sum_{\ell\in\Ls}\p(\ell|m,w_{b})\geq\sum_{\ell\in\Ls}\p(\ell|m,w_{a}).
\end{equation*}
Equation \eqref{eq_assum1_G} is obtained by Equations \eqref{prop:assum:eq:1} and \eqref{prop:assum:eq:2}.
\qed

\subsection{Almost Submodular Maximization}\label{Nemhauser:supp}
In this Section we provide three Lemmas: Lemma \ref{Nemhauser} is the main result which generalizes the classical result proposed in \cite{nemhauser1978analysis} to "almost"-monotone and "almost"-submodular functions. 
\begin{lemma}\label{Nemhauser}
Let $g:2^L \to \R^+$ be a function mapping subsets of $L$ to non-negative reals with the following properties:
\begin{enumerate} 
\item $g(\emptyset)=0$ 
\item for all $\w \subset L, \ell \in L$,  $g(\w\cup \ell)\geq g(\w)-\subeps$
\item for all $\w_a \subseteq \w_b \subseteq L$ and $\ell \in L$, 
		$$g(\w_{a}\cup \ell)-g(\w_{a})\geq g(\w_{a}\cup \ell)-g(\w_{a})-\theta$$ 
		for some scalar $\theta$.
\end{enumerate}
			 Then, it is obtained that
		\begin{equation*}
		g(\w_{k})\geq\TBb\left(\max_{\w\in\Ls^{k}} g(\w)-(k-1)\theta-k\subeps\right),
		\end{equation*}
		where $w_{k}\in\Ls^{k}$ is obtained by the Greedy Algorithm and
		\begin{equation}\label{eq:TBb}
		\TBb=1-\left(1-\frac{1}{k}\right)^{k+1} \geq 1-\frac{1}{e}
		\end{equation}

	\end{lemma}
	\noindent{\bf Proof\ignore{ of Lemma \ref{Nemhauser}}: (Based on Nemhauser et al. 1978)\ } 
	By Lemma \ref{lemma:new:sub:mod:2}, for $i+1=k$ we have
	\begin{equation*}
	g(\w_{k})\geq \left(1-\left(1-\frac{1}{k}\right)^{k}\right)\left(g(OPT)-\theta (k-1)-k\subeps\right)\ ,
	\end{equation*}
	where $\w_{k}$ is the set that obtained by the greedy Algorithm after $k$ iterations and the set $OPT$ attains the optimal value, namely, $\{OPT\}=\arg\max_{\w\in\Ls^{k}}g(\w)$.
	\qed
\ignore{\proof
	The proof is provided in \S\ref{Nemhauser:supp} in the supplementary material.
	\qed}

In the following Lemma we bound the loss of adding greedily one item to a given set. This lemma is used for the proof of Lemma \ref{lemma:new:sub:mod:2} (which is used for the proof of Lemma \ref{Nemhauser}).
\begin{lemma}\label{lemma:new:sub:mod:1}
Under the conditions of Lemma \ref{Nemhauser}, after applying the Greedy Algorithm, it holds that
\begin{equation*}
g(w_{i+1})-g(w_{i})\geq\frac{1}{k}\left(g(OPT)-g(w_{i})\right)-\frac{\theta(k-1)}{k}-\subeps\ ,
\end{equation*}
where the set $OPT$ attains the optimal value, namely, $\{OPT\}=\arg\max_{\w\in\Ls^{k}}g(\w)$, and the set $\w_{i}$ is the set that obtained by the greedy Algorithm after $i$ iterations.
\end{lemma}
\proof
For every set of content items $T=\{\ell_{1},...,\ell_{|T|}\}$ and $j\leq|T|$, we denote $T_{j}=\{\ell_{1},...,\ell_{j}\}$ and $T_{0}=\emptyset$. So, we have,
\begin{equation*}
g(w_{i}\cup T)-g(w_{i})=\sum_{j=1}^{|T|}g(w_{i}\cup T_{j})-g(w_{i}\cup T_{j-1})\ .
\end{equation*}
Then, since for every $j\geq2$ and $\ell\in\Ls$
\begin{equation*}
g(w_{i+1})-g(w_{i})\geq g(w_{i}\cup \ell)-g(w_{i})\ ,
\end{equation*}
and
\begin{equation*}
g(w_{i+1})-g(w_{i})\geq g(w_{i}\cup T_{j-1}\cup \ell)-g(w_{i}\cup T_{j-1})-\theta\ ,
\end{equation*}
it is obtained that
\begin{equation*}
|T|\left(g(w_{i+1})-g(w_{i})\right)\geq g(w_{i}\cup T)-g(w_{i})-(|T|-1)\theta\ .
\end{equation*}
Therefore,
\begin{equation*}
g(w_{i+1})-g(w_{i})\geq \frac{g(w_{i}\cup T)-g(w_{i})-(|T|-1)\theta}{|T|}\triangleq\Phi_{\Delta}\ .
\end{equation*}
Then, for the choice of $T=OPT\setminus w_{i}$,  since $|T|\leq k$, we have
\begin{equation*}
\Phi_{\Delta}\geq\frac{g(w_{i}\cup T)-g(w_{i})}{k}-\frac{\theta(k-1)}{k}-\subeps=\frac{g(OPT)-g(w_{i})}{k}-\frac{\theta(k-1)}{k}-\subeps\ .
\end{equation*}
\qed

In the following Lemma we bound the loss that is incurred by adding greedily a certain number of items to a set. This lemma is used for the proof of Lemma \ref{Nemhauser}.

\begin{lemma}\label{lemma:new:sub:mod:2}
Under the Greedy Algorithm, it holds that
\begin{equation*}
g(w_{i+1})\geq\left(1-\left(1-\frac{1}{k}\right)^{i+1}\right)\left(g(OPT)-\theta(k-1)-k\subeps\right)\ ,
\end{equation*}
where the set $OPT$ attains the optimal value, namely, $\{OPT\}=\arg\max_{\w\in\Ls^{k}}g(\w)$, and the set $\w_{i}$ is the set that obtained by the greedy Algorithm after $i$ iterations.
\end{lemma}
\proof
We prove this claim by induction over $i$. Since we assume that $g(\emptyset)=0$, the base case, for $i=0$ can be derived from Lemma \ref{lemma:new:sub:mod:1}. For $i>0$, it is obtained by Lemma \ref{lemma:new:sub:mod:1} that
\begin{equation*}
g(\w_{i+1})\geq\frac{1}{k}g(OPT)-\frac{\theta (k-1)}{k}-\subeps+\frac{k-1}{k}g(\w_{i})\triangleq\Upsilon\ .
\end{equation*}
Then, by the induction assumption,
\begin{equation*}
\begin{aligned}
\Upsilon\geq&\frac{1}{k}g(OPT)-\frac{\theta (k-1)}{k}-\subeps+\frac{k-1}{k}\left(1-\left(1-\frac{1}{k}\right)^{i}\right)\left(g(OPT)-\theta(k-1)-k\subeps\right)
\\&=\left(1-\left(1-\frac{1}{k}\right)^{i+1}\right)\left(g(OPT)-\theta(k-1)-k\subeps\right)
\end{aligned}\ .
\end{equation*}
\qed

\section{Proof of Theorem \ref{thm:g}}\label{thm:g:supp}
In this Section we provide the proof of Theorem \ref{thm:g}. Here, we use $\V^{t}(\C)$ and $\Q^{t}(\w,\C)$ for shorthand of $\left(\gT\right)^{t} \V(\C)$ and $\Q_{\V^{t-1}}(\w,\C)$, respectively. We begin with a Lemma that upper bounds the value function obtained by the $\gT$ operator. Then, in Lemmas \ref{lem:new:AV>BV_G base} and \ref{lem:new:AV>BV_G} we show a monotonic increasing property of the value function. In Lemmas \ref{lem:new:2:AV>BV_G base} and \ref{lem:new:2:AV>BV_G}  we show the convexity of the value function. In Lemma \ref{lem:new:2:submodularity} we show the ``almost"- submodularity of the Q-function, while in Lemma \ref{lem:new:2:monotonicity} we show the monotonicity of the Q-function. Lemma \ref{lem:theta_bound} shows the direct relation between a larger set of items larger long term cumulative reward. We conclude this section with the proof of Theorem \ref{thm:g} which is based on Lemmas \ref{Bound_gv}-\ref{lem:theta_bound}.

\begin{lemma}\label{Bound_gv}
For every $\C\in\SC$, $t\geq0$ and zero initiation of the value function (namely, $\gV^{0}=0$), it holds that $$\gV^{t}(\C)\leq\frac{1}{B-\gamma}\ .$$
\end{lemma}
\proof
It is obtained easily by Assumption \ref{assum:sumP_B}, that for every $\C\in\SC$ it holds that
\begin{equation*}
\gV^{t}(\C)\leq\frac{1}{B}\left(1+\gamma\overline{\V}_{\text{greedy}}^{t-1}\right)\ ,
\end{equation*}
where $\overline{\V}_{\text{greedy}}$ is an upper bound on $\gV(\C)$ for every $\C\in\SC$.
So, since $\gV^{0}=0$, we have that
$$\gV^{t}(\C)\leq\frac{1}{B-\gamma}\ .$$
\qed

In the next two lemmas we show a monotonic increasing property of the value function that is obtained by the operator $\gT$.

\begin{lemma} \label{lem:new:AV>BV_G base}
Let $\C^{1},\C^{2}\in\SC$ and let $A_{1}$ and $B_{2}$ be a pair of positive constants. Assume that 
\begin{equation}
A_{1}\C^{1}(m)\geq B_{2}\C^{2}(m)\ ,
\end{equation}
for all $m \in \M$. Then it holds that 
$$A_{1}\gV^{0}(\C^{1})\geq B_{2}\gV^{0}(\C^{2})\ .$$
\end{lemma}
\begin{proof}
The result is immediate since $\gV^{0}(\C)=0$ for every $\C\in\SC$.
\qed
\end{proof}
\begin{lemma}\label{lem:new:AV>BV_G}
Let $\C^{1},\C^{2}\in\SC$ and let $A_{1}$ and $B_{2}$ be a pair of positive constants. Assume that 
\begin{equation}\label{lem:new:AV>BV:cond_G}
A_{1}\C^{1}(m)\geq B_{2}\C^{2}(m)\ ,
\end{equation}
for all $m \in \M$. Then, we have for any positive integer $t$
$$A_{1}\gV^{t}(\C^{1})\geq B_{2}\gV^{t}(\C^{2})\ .$$
\end{lemma}
\begin{proof}
We prove the claim by induction over $t$. The base case for $t=0$ holds due to Lemma~\ref{lem:new:AV>BV_G base}. Assume that the lemma is satisfied for $t-1$.
Recall Equation \eqref{eq:posterior_b} characterizing $\C'_{\ell,w,\C}(m')$
$$ \C'_{\ell,w,\C}(m')=\frac{\C(m')\p(\ell|m',w)}{\sum_{m\in\M}\C(m)\p(\ell|m,w)}\ . $$
By plugging in with Equation \eqref{lem:new:AV>BV:cond_G} we get that
$$A_{1}\left(\sum_{m\in\M}\C^{1}(m)\p(\ell|m,w)\right)\C'_{\ell,w,\C^{1}}(m')\geq B_{2}\left(\sum_{m\in\M}\C^{2}(m)\p(\ell|m,w)\right)\C'_{\ell,w,\C^{2}}(m')\ ,$$
for any $\w\in\Ls^{\K}$, $\ell\in\Ls$ and $m'\in\M$, as $\p(\ell|m',w) \geq 0$. Therefore, by the induction assumption applied for 
$$A_1' = A_{1}\left(\sum_{m\in\M}\C^{1}(m)\p(\ell|m,w)\right) \ \ \ , \ \ \ B_2' = B_{2}\left(\sum_{m\in\M}\C^{2}(m)\p(\ell|m,w)\right) \ ,$$
\begin{equation}\label{lem:new:AV>BV_G:1}
A_{1}\sum_{m\in\M}\C^{1}(m)\p(\ell|m,w)\gV^{t-1}\left(\C'_{\ell,w,\C^{1}}\right)\geq
B_{2}\sum_{m\in\M}\C^{2}(m)\p(\ell|m,w)\gV^{t-1}\left(\C'_{\ell,w,\C^{2}}\right)\ ,
\end{equation}
for every $\ell\in\Ls$ and $\w\in\Ls^{\K}$. Furthermore, by Equation \eqref{lem:new:AV>BV:cond_G}
\begin{equation}\label{lem:new:AV>BV_G:2}
A_{1}\sum_{m\in\M}\C^{1}(m)\p(\ell|m,w)\geq
B_{2}\sum_{m\in\M}\C^{2}(m)\p(\ell|m,w)\ ,
\end{equation}	
for every $\ell\in\Ls$ and $\w\in\Ls^{\K}$. So, by the fact that
$$
A_{1}\gQ^{t}(w,\C^{1})=A_{1}\sum_{m\in\M}\sum_{\ell\in\Ls}\C^{1}(m)\p(\ell|m,w)\left(1+\gamma\gV^{t-1}\left(\C'_{\ell,w,\C^{1}}\right)\right)\ ,
$$
and also respectively for $B_{2}$ and $\C^{2}$,
it is obtained by Equations \eqref{lem:new:AV>BV_G:1} and \eqref{lem:new:AV>BV_G:2} that
\begin{equation}\label{lem:new:AV>BV:Q:eq_G:finall}
A_{1}\gQ^{t}(w,\C^{1})\geq B_{2}\gQ^{t}(w,\C^{2})\ ,
\end{equation}
for any $w\in\Ls^{\K}$.
So, by Definition \ref{def:greedyV} the result is obtained.
\qed
\end{proof}

In the following two lemmas we show a convexity property of the value function that is  obtained by the $\gT$ operator.
\begin{lemma} \label{lem:new:2:AV>BV_G base}
	Let $\C,\C^{1},\C^{2}\in\SC$ and let $A$, $B_{1}$ and $B_{2}$ be a tuple of positive constants.  Assume that 
	\begin{equation}
	A\C(m)= B_{1}\C^{1}(m)+B_{2}\C^{2}(m)\ ,
	\end{equation}
	for all $m \in \M$. Then it holds that
	$$A\gV^{0}(\C)\leq B_{1}\gV^{0}(\C^{1})+B_{2}\gV^{0}(\C^{2})\ .$$
\end{lemma}
\begin{proof}
	True for initiate value function $\gV^{0}(\C)=0$ for every $\C\in\SC$.
	\qed
\end{proof}

\begin{lemma} [Convexity] \label{lem:new:2:AV>BV_G}
	Let $\C,\C^{1},\C^{2}\in\SC$ and let $A$, $B_{1}$ and $B_{2}$ be a tuple of positive constants. Assume that 
	\begin{equation}\label{lem:new:2:AV>BV:cond_G}
	A\C(m)= B_{1}\C^{1}(m)+B_{2}\C^{2}(m)\ ,
	\end{equation}
	for all $m \in \M$. We have that for any positive integer $t$ it holds that
	$$A\gV^{t}(\C)\leq B_{1}\gV^{t}(\C^{1})+B_{2}\gV^{t}(\C^{2})\ .$$
\end{lemma}

\begin{proof}
	We prove the claim by induction over $t$. The base case for $t=0$ holds due to Lemma~\ref{lem:new:2:AV>BV_G base}. Assume that the lemma is satisfied for $t-1$.
	Recall Equation \eqref{eq:posterior_b} characterizing $\C'_{\ell,w,\C}(m')$
	$$ \C'_{\ell,w,\C}(m')=\frac{\C(m')\p(\ell|m',w)}{\sum_{m\in\M}\C(m)\p(\ell|m,w)}\ . $$
	By plugging in with Equation \eqref{lem:new:2:AV>BV:cond_G} we get that
	\begin{equation*}
	\begin{aligned}
	&A\left(\sum_{m\in\M}\C(m)\p(\ell|m,w)\right)\C'_{\ell,w,\C}(m')=
	\\& B_{1}\left(\sum_{m\in\M}\C^{1}(m)\p(\ell|m,w)\right)\C'_{\ell,w,\C^{1}}(m')+B_{2}\left(\sum_{m\in\M}\C^{2}(m)\p(\ell|m,w)\right)\C'_{\ell,w,\C^{2}}(m')
	\end{aligned}\ ,
	\end{equation*}
	for any $\w\in\Ls^{\K}$, $\ell\in\Ls$ and $m'\in\M$, as $\p(\ell|m',w) \geq 0$. Therefore, by the induction assumption applied for 
	\begin{align*}
	&A' = A\left(\sum_{m\in\M}\C(m)\p(\ell|m,w)\right) \ \ \ \ \ , \ \ \ B_1' = B_{1}\left(\sum_{m\in\M}\C^{1}(m)\p(\ell|m,w)\right) \ \ \ , \ \ \ \\&B_2' = B_{2}\left(\sum_{m\in\M}\C^{2}(m)\p(\ell|m,w)\right) \ ,
	\end{align*}
	\begin{equation}\label{lem:new:2:AV>BV_G:1}
	\begin{aligned}
	A\sum_{m\in\M}\C(m)\p(\ell|m,w)\gV^{t-1}\left(\C'_{\ell,w,\C}\right)\leq &B_{1}\sum_{m\in\M}\C^{1}(m)\p(\ell|m,w)\gV^{t-1}\left(\C'_{\ell,w,\C^{1}}\right)
	\\&+
	B_{2}\sum_{m\in\M}\C^{2}(m)\p(\ell|m,w)\gV^{t-1}\left(\C'_{\ell,w,\C^{2}}\right)
	\end{aligned}\ ,
	\end{equation}
	for every $\ell\in\Ls$ and $\w\in\Ls^{\K}$. Furthermore, by Equation \eqref{lem:new:2:AV>BV:cond_G}
	\begin{equation}\label{lem:new:2:AV>BV_G:2}
	A\sum_{m\in\M}\C(m)\p(\ell|m,w)=B_{1}\sum_{m\in\M}\C^{1}(m)\p(\ell|m,w)+
	B_{2}\sum_{m\in\M}\C^{2}(m)\p(\ell|m,w)\ ,
	\end{equation}	
	for every $\ell\in\Ls$ and $\w\in\Ls^{\K}$. So, by the fact that
	$$
	A\gQ^{t}(w,\C)=A\sum_{m\in\M}\sum_{\ell\in\Ls}\C(m)\p(\ell|m,w)\left(1+\gamma\gV^{t-1}\left(\C'_{\ell,w,\C}\right)\right)\ ,
	$$
	and also respectively for $B_{1}$, $\C^{1}$, $B_{2}$ and $\C^{2}$,
	it is obtained by Equations \eqref{lem:new:2:AV>BV_G:1} and \eqref{lem:new:2:AV>BV_G:2} that
	\begin{equation}\label{lem:new:2:AV>BV:Q:eq_G:finall}
	A\gQ^{t}(w,\C)\leq B_{1}\gQ^{t}(w,\C^{1})+B_{2}\gQ^{t}(w,\C^{2})\ ,
	\end{equation}
	for any $w\in\Ls^{\K}$.
	So, by Definition \ref{def:greedyV} the result is obtained.
	\qed
\end{proof}

In the following lemma we show that the Q-function obtained by the $\gT$ operator is "almost"-submodular.

\begin{lemma} [Submodularity] \label{lem:new:2:submodularity}
	For any positive integer $t$, where $\w_{b}\supset\w_{a}$ and $\ell'\not\in w_{b}$ it holds that
	\begin{equation*}
	\gQ^{t}(\{w_{a}\cup \ell'\},\C)-\gQ^{t}(w_{a},\C)\geq\gQ^{t}(\{w_{b} \cup \ell'\},\C)-\gQ^{t}(w_{b},\C) -\theta(\ell',w_{a},\C)\ ,
	\end{equation*}
	where $$\theta(\ell',w_{a},\C)= \sum_{m\in\M}\C(m)P(\ell'|m,w_{a}\cup \ell')\sum_{\ell\in\Ls}P(\ell|m,w_{a})\frac{\gamma}{B-\gamma}\ .$$
\end{lemma}
\proof
Let
\begin{equation}
\gQ^{t}(\{w_{a}\cup \ell'\},\C)-\gQ^{t}(w_{a},\C)-\left(\gQ^{t}(\{w_{b}\cup \ell'\},\C)-\gQ^{t}(w_{b},\C)\right)=\Phi_{1}+\Phi_{2}^{1}+\Phi_{2}^{2}\ ,
\end{equation}
where
\begin{align*}
&\begin{aligned}
\Phi_{1}\triangleq&\sum_{m\in\M}\C(m)\sum_{\ell\in\Ls}P(\ell|m,w_{a}\cup \ell')-\sum_{m\in\M}\C(m)\sum_{\ell\in\Ls}P(\ell|m,w_{a})
\\&+\sum_{m\in\M}\C(m)\sum_{\ell\in\Ls}P(\ell|m,w_{b})-\sum_{m\in\M}\C(m)\sum_{\ell\in\Ls}P(\ell|m,w_{b}\cup \ell')
\end{aligned}\ ,\\
&\begin{aligned}
\Phi_{2}^{1}\triangleq&\sum_{m\in\M}\C(m)P(\ell' | m,w_{a}\cup \ell')\gamma\gV(\C'_{\ell',\{w_{a}\cup \ell'\},\C})
\\&+\sum_{m\in\M}\C(m)\sum_{\ell\in\Ls}P(\ell|m,w_{b})\gamma\gV(\C'_{\ell,w_{b},\C})
\\&-\sum_{m\in\M}\C(m)\sum_{\ell\in\Ls}P(\ell|m,w_{b}\cup \ell')\gamma\gV(\C'_{\ell,\{w_{b}\cup \ell'\},\C})
\end{aligned}\ ,\\
&\text{and} \nonumber\\
&\begin{aligned}
\Phi_{2}^{2}\triangleq&\sum_{m\in\M}\C(m)\sum_{\ell\in\Ls\setminus \ell'}P(\ell|m,w_{a}\cup \ell')\gamma\gV(\C'_{\ell,\{w_{a}\cup \ell'\},\C})
\\&-\sum_{m\in\M}\C(m)\sum_{\ell\in\Ls}P(\ell|m,w_{a})\gamma\gV(\C'_{\ell,w_{a},\C})
\end{aligned}\ .
\end{align*}

Then by Proposition \ref{assum1_G} it is obtained that 
\begin{equation}\label{eq:lem:new:2:submodularity:Phi1_1:bound}
\Phi_{1}\geq 0\ .
\end{equation}
For bounding $\Phi_{2}^{1}$ we note that according to the definition of $\C'_{\ell,w,\C}$ and Assumption \ref{assum:w_w'}, it is obtained for every $\ell\in\w_{b}$ that
\begin{equation*}
\sum_{m\in\M}\C(m)P(\ell|m,w_{b})\gamma\C'_{\ell,w_{b},\C}(m')
\geq\sum_{m\in\M}\C(m)P(\ell|m,w_{b}\cup \ell')\gamma\C'_{\ell,\{w_{b}\cup \ell'\},\C}(m')\ ,
\end{equation*}
and for $\ell'$ that
\begin{equation*}
\sum_{m\in\M}\C(m)P(\ell'|m,w_{a}\cup \ell')\gamma\C'_{\ell',\{w_{a}\cup \ell'\},\C}(m')
\geq\sum_{m\in\M}\C(m)P(\ell'|m,w_{b}\cup \ell')\gamma\C'_{\ell',\{w_{b}\cup \ell'\},\C}(m')\ ,
\end{equation*}
for every $m'\in\M$.
Therefore, by Lemma \ref{lem:new:AV>BV_G}, for every $\ell\in\w_{b}$ it is obtained that
\begin{equation*}
\sum_{m\in\M}\C(m)P(\ell|m,w_{b})\gamma\gV^{t-1}(\C'_{\ell,w_{b},\C})
\geq\sum_{m\in\M}\C(m)P(\ell|m,w_{b}\cup \ell')\gamma\gV^{t-1}(\C'_{\ell,\{w_{b}\cup \ell'\},\C})\ ,
\end{equation*}
and for $\ell'$ it is obtained that
\begin{equation*}
\sum_{m\in\M}\C(m)P(\ell'|m,w_{a}\cup \ell')\gamma\gV(\C'_{\ell',\{w_{a}\cup \ell'\},\C})
\geq\sum_{m\in\M}\C(m)P(\ell'|m,w_{b}\cup \ell')\gamma\gV^{t-1}(\C'_{\ell',\{w_{b}\cup \ell'\},\C})\ .
\end{equation*}
So,
\begin{equation}\label{eq:lem:new:2:submodularity:Phi1_2:bound}
\Phi_{2}^{1}\geq 0\ .
\end{equation}
In addition, we note that for every $\ell\in\w_{a}$ it is obtained by Assumption \ref{assum:w_w'} that
\begin{equation*}
\begin{aligned}
\sum_{m\in\M}\C(m)P(\ell|m,w_{a})\gamma\C'_{\ell,w_{a},\C}=&\sum_{m\in\M}\C(m)P(\ell|m,w_{a}\cup \ell')\gamma\C'_{\ell,\{w_{a}\cup \ell'\},\C}
\\&+\sum_{m\in\M}\C(m)P(\ell'|m,w_{a}\cup \ell')P(\ell|m,w_{a})\gamma\tilde{\C}\ ,
\end{aligned}
\end{equation*}
where
\begin{equation*}
\tilde{\C}(m')=\frac{\C(m')P(\ell'|m',w_{a}\cup \ell')P(\ell|m',w_{a})}{\sum_{m\in\M}\C(m)P(\ell'|m,w_{a}\cup \ell')P(\ell|m,w_{a})}\ .
\end{equation*}
So, by Lemmas \ref{lem:new:2:AV>BV_G} and \ref{Bound_gv},
\begin{equation}\label{eq:lem:new:2:submodularity:Phi2_2:bound}
\Phi_{2}^{2}\geq -\sum_{m\in\M}\C(m)P(\ell'|m,w_{a}\cup \ell')\sum_{\ell\in\Ls}P(\ell|m,w_{a})\gamma\frac{1}{B-\gamma}\ .
\end{equation}
Therefore, by Equations \eqref{eq:lem:new:2:submodularity:Phi1_1:bound}, \eqref{eq:lem:new:2:submodularity:Phi1_2:bound} and \eqref{eq:lem:new:2:submodularity:Phi2_2:bound} it is obtained that
\begin{equation*}
\gQ^{t}(\{w_{a}\cup \ell'\},\C)-\gQ^{t}(w_{a},\C)\geq\gQ^{t}(\{w_{b}\cup \ell'\},\C)-\gQ^{t}(w_{b},\C) -\theta(\ell',w_{a},\C)\ ,
\end{equation*}
where 
\begin{equation*}
\begin{aligned}
\theta(\ell',w_{a},\C)= &\sum_{m\in\M}\C(m)P(\ell'|m,w_{a}\cup \ell')\sum_{\ell\in\Ls}P(\ell|m,w_{a})\frac{\gamma}{B-\gamma}
\end{aligned}\ .
\end{equation*}
\qed

In the following lemma we show that the Q-function obtained by the $\gT$ operator is monotone.

\begin{lemma} [Monotonicity] \label{lem:new:2:monotonicity}
	If $B\geq1+\gamma$, then for any $\C\in\SC$, a set of content items $\w$ such that $\ell'\not\in\w$ and $t\geq0$ it holds that $\gQ(\{\w\cup \ell'\},\C)\geq\gQ(\w,\C)$. Where $B$ is the constant in Assumption \ref{assum:sumP_B}.
\end{lemma}
\proof
\begin{equation*}
\gQ(\{\w\cup \ell'\},\C)-\gQ(\w,\C)=\Psi_{1}+\Psi_{2}\ ,
\end{equation*}
where
\begin{equation*}
\Psi_{1}\triangleq\sum_{m\in\M}\C(m)\sum_{\ell\in\Ls}P(\ell|m,w\cup \ell')-\sum_{m\in\M}\C(m)\sum_{\ell\in\Ls}P(\ell|m,w)\ ,
\end{equation*}
and
\begin{equation*}
\Psi_{2}\triangleq\sum_{m\in\M}\C(m)\sum_{\ell\in\Ls}P(\ell|m,w\cup \ell')\gamma\gV(\C'_{\ell,\{w\cup \ell'\},\C})-\sum_{m\in\M}\C(m)\sum_{\ell\in\Ls}P(\ell|m,w)\gamma\gV(\C'_{\ell,w,\C})\ .
\end{equation*}
Then, by Assumption \ref{assum:w_w'} it is obtained that
\begin{equation}\label{eq:lemm:new:2:Phi1:bound}
\Psi_{1}=\sum_{m\in\M}\C(m)P(\ell'|m,w\cup \ell')\left(1-\sum_{\ell\in\Ls}P(\ell|m,w)\right)\ .
\end{equation}
For bounding $\Psi_{2}$, recall Equation \eqref{eq:posterior_b} characterizing $\C'_{\ell,w,\C}(m')$
$$ \C'_{\ell,w,\C}(m')=\frac{\C(m')\p(\ell|m',w)}{\sum_{m\in\M}\C(m)\p(\ell|m,w)}\ . $$
Then, for every $\ell\in\w$ it is obtained by Assumption \ref{assum:w_w'} that
\begin{equation*}
\begin{aligned}
\sum_{m\in\M}\C(m)P(\ell|m,w)\gamma\C'_{\ell,w,\C}=&\sum_{m\in\M}\C(m)P(\ell|m,w\cup \ell')\gamma\C'_{\ell,\{w\cup \ell'\},\C}
\\&+\sum_{m\in\M}\C(m)P(\ell'|m,w\cup \ell')P(\ell|m,w)\gamma\tilde{\C}\ ,
\end{aligned}
\end{equation*}
where
\begin{equation*}
\tilde{\C}(m')=\frac{\C(m')P(\ell'|m',w\cup \ell')P(\ell|m',w)}{\sum_{m\in\M}\C(m)P(\ell'|m,w\cup \ell')P(\ell|m,w)}\ .
\end{equation*}
So, by Lemmas \ref{lem:new:2:AV>BV_G} and \ref{Bound_gv},
\begin{equation}\label{eq:lem:new:2:mono}
\begin{aligned}
\sum_{m\in\M}\C(m)P(\ell|m,w)\gamma\gV\left(\C'_{\ell,w,\C}\right)\leq&\sum_{m\in\M}\C(m)P(\ell|m,w\cup \ell')\gamma\gV\left(\C'_{\ell,\{w\cup \ell'\},\C}\right)
\\&+\sum_{m\in\M}\C(m)P(\ell'|m,w\cup \ell')P(\ell|m,w)\gamma\frac{1}{B-\gamma}
\end{aligned}\ .
\end{equation}
Therefore, by Equation \ref{eq:lem:new:2:mono} it is obtained that
\begin{equation}\label{eq:lemm:new:2:Phi2:bound}
\Psi_{2}\geq-\sum_{m\in\M}\C(m)\sum_{\ell\in\Ls}P(\ell'|m,w\cup \ell')P(\ell|m,w)\gamma\frac{1}{B-\gamma}\ .
\end{equation}
So, by Equations \eqref{eq:lemm:new:2:Phi1:bound} and \eqref{eq:lemm:new:2:Phi2:bound},
\begin{equation*}
\Psi_{1}+\Psi_{2}\geq\sum_{m\in\M}\C(m)P(\ell'|m,w\cup \ell')\left(1-\frac{B}{B-\gamma}\sum_{\ell\in\Ls}P(\ell|m,w)\right)\ .
\end{equation*}
Then, by Assumption \ref{assum:sumP_B}, it is obtained that $\Psi_{1}+\Psi_{2}\geq0$ for $B\geq1+\gamma$, and therefore the Lemma holds.
\qed
 
The following lemma serves us to show that as the set of items is larger the long term cumulative reward is larger.

\begin{lemma}\label{lem:theta_bound}
	For every content item $\ell'$, state $\C\in\SC$, a set $\w_{a}$ that contains less than $\K$ content items and a set $\w_{b}\supseteq\w_{a}$ that contains $\K$ content items, if $B\geq 2$ then, it holds that
	\begin{equation}\label{eq:lem:theta_bound}
	\sum_{m\in\M}\C(m)P(\ell'|m,\w_{a}\cup \ell')\sum_{\ell\in\Ls}P(\ell|m,\w_{a})\leq \sum_{m\in\M}\C(m)P(\ell'|m,\w_{b}\cup \ell')\sum_{\ell\in\Ls}P(\ell|m,\w_{b})\ .
	\end{equation}
\end{lemma}

\proof
First, lets address the case in which $\w_{b}=\{\w_{a}\cup \ell_{b}\}$ for some $\ell_{b}\in\Ls$. By assumption \ref{assum:w_w'} and proposition \ref{assum1_G} it holds that,
\begin{equation}\label{eq:lem:theta_bound:1}
P(\ell'|m,\w_{a}\cup \ell')=P(\ell'|m,\w_{b}\cup \ell')\sum_{i=0}^{\infty}\left(P(\ell_{b}|m,\w_{b}\cup \ell')\right)^{i}\leq P(\ell'|m,\w_{b}\cup \ell')\sum_{i=0}^{\infty}\left(P(\ell_{b}|m,\w_{b})\right)^{i}\ ,
\end{equation}
and that
\begin{equation}\label{eq:lem:theta_bound:2}
\sum_{\ell\in\Ls}P(\ell|m,\w_{a})=\sum_{\ell\in\Ls\setminus \ell_{b}}P(\ell|m,\w_{b})\sum_{i=0}^{\infty}\left(P(\ell_{b}|m,\w_{b})\right)^{i}\ ,
\end{equation}
for every $m\in\M$. In addition, by assumption \ref{assum:sumP_B} it is obtained that
\begin{equation*}
\sum_{\ell\in\Ls\setminus \ell_{b}}P(\ell|m,\w_{b})\leq\frac{1}{B}-P(\ell_{b}|m,\w_{b})=\frac{1-BP(\ell_{b}|m,\w_{b})}{B}\ .
\end{equation*}
So,
\begin{equation*}
\left(\frac{BP(\ell_{b}|m,\w_{b})}{1-BP(\ell_{b}|m,\w_{b})}+1\right)\sum_{\ell\in\Ls\setminus \ell_{b}}P(\ell|m,\w_{b})\leq\sum_{\ell\in\Ls}P(\ell|m,\w_{b})\ .
\end{equation*}
Then, since $P(\ell_{b}|m,\w_{b})\leq\frac{1}{B}$, for $B\geq 2$, after some algebraic calculations, it is obtained that
\begin{equation}\label{eq:lem:theta_bound:3}
\left(\sum_{i=0}^{\infty}\left(P(\ell_{b}|m,\w_{b})\right)^{i}\right)^{2}\leq \frac{BP(\ell_{b}|m,\w_{b})}{1-BP(\ell_{b}|m,\w_{b})}+1\ .
\end{equation}
Therefore, by Equations \eqref{eq:lem:theta_bound:1}, \eqref{eq:lem:theta_bound:2} and \eqref{eq:lem:theta_bound:3}, Equation \eqref{eq:lem:theta_bound} holds for $\w_{b}=\{\w_{a}\cup \ell_{b}\}$ and states in which one user type is in probability of $1$, namely, $\C=(0,...,1,...,0)$.

The case in which $\w_{b}$ is larger than $\w_{a}$ by more than one content item, can be addressed by induction, with the above as the induction step. Then, since Equation \eqref{eq:lem:theta_bound} holds for every state of the type $\C=(0,...,1,...,0)$ it holds for every $\C\in\SC$.
\qed

\ignore{
for every $m\in\M$. In addition, by assumption \ref{assum:sumP_B} and proposition \ref{assum1_G} it is obtained that
\begin{equation*}
\sum_{\ell\in\Ls\setminus \ell_{b}}P(\ell|m,\w_{b})\leq\sum_{\ell\in\Ls}P(\ell|m,\w_{b})\leq \frac{1}{B}
\end{equation*}
So,
\begin{equation*}
\left(BP(\ell_{b}|m,\w_{b})+1\right)\sum_{\ell\in\Ls\setminus \ell_{b}}P(\ell|m,\w_{b})\leq\sum_{\ell\in\Ls}P(\ell|m,\w_{b})
\end{equation*}
Since $P(\ell_{b}|m,\w_{b})\leq\frac{1}{B}$, for $B\geq 2+\sqrt{2}$ it holds that
\begin{equation}\label{eq:lem:theta_bound:3}
\left(\sum_{i=0}^{\infty}\left(P(\ell_{b}|m,\w_{b})\right)^{i}\right)^{2}\leq BP(\ell_{b}|m,\w_{b})+1
\end{equation}
Therefore, by Equations \eqref{eq:lem:theta_bound:1}, \eqref{eq:lem:theta_bound:2} and \eqref{eq:lem:theta_bound:3}, Equation \eqref{eq:lem:theta_bound} holds.

The case in which $\w_{b}$ is larger than $\w_{a}$ by more than one content item, can be addressed by induction, with the above as the induction step. So the Lemma holds 
\qed
}

\noindent{\bf Proof of Theorem \ref{thm:g}:\ } 

\proof
Since the value obtained by the maximization in Equation \eqref{eq:def:greedyV} is equal or smaller than the accurate maximal value we have that
\begin{equation}\label{eq:thm:g:1}
(\gT \V)(\C)\leq (\T \V)(\C).
\end{equation}
So, Equation \eqref{thm:greedy:first} holds for $t=1$. Now, let's assume that Equation \eqref{thm:greedy:first} holds for $t-1$.
Then, by the monotonicity of the original DP operator (namely, $\T$) it is obtained that
\begin{equation*}
(\T\gT^{t-1} \V)(\C)\leq (\T\T^{t-1} \V)(\C),
\end{equation*}
but for $(\gT^{t-1} \V)(\C)=(\V)(\C)$, by Equation \eqref{eq:thm:g:1}
\begin{equation*}
(\gT\gT^{t-1} \V)(\C)\leq(\T\gT^{t-1} \V)(\C).
\end{equation*}
So, Equation \eqref{thm:greedy:first} holds for $t$, and therefore by induction Equation \eqref{thm:greedy:first} holds for every $\geq1$.

Now we prove Equation \eqref{eq:ind:assum} by induction. We note that
by the fact that $\gQ(\emptyset,\C)=0$ and by Lemmas \ref{lem:new:2:submodularity}, \ref{lem:new:2:monotonicity} and \ref{lem:theta_bound}, which are provided and proved in Section \ref{thm:g:supp}  in the supplementary material, it is obtained that Lemma \ref{Nemhauser} can be applied on the operator $\gT$, with $\overline{\theta}(\C)$ as defined in Equation \eqref{THM:Theta}. So, by Lemma \ref{Nemhauser} we have
\begin{equation}\label{eq:thm:g}
(\T \V)(\C)\leq \frac{1}{\TBb}(\gT \V)(\C)+(k-1)\overline{\theta}(\C).
\end{equation}
In addition, we note that
\begin{equation}\label{for:ind:step1}
(\T\beta \V)(\C)=(\T_{\beta\gamma} \V)(\C),
\end{equation}
and that
\begin{equation}\label{for:ind:step2}
\TBb\V(\C)\leq\V(\C).
\end{equation}
So, by Equations \eqref{eq:thm:g}, \eqref{for:ind:step1} and \eqref{for:ind:step2} we have that
\begin{equation}\label{eq:ind:assum:0}
\TBb\left((\T_{\TBb\gamma} \V)(\C) -(k-1)\overline{\theta}(\C)\right)\leq (\gT\V)(\C).
\end{equation}
So, by Equation \eqref{eq:ind:assum:0}, we can easily see that Equation \eqref{eq:ind:assum} satisfies for $t=1$.
Now, Let's assume that Equation \eqref{eq:ind:assum} satisfies for $t-1$.

By the fact that 
\begin{equation*}
\begin{aligned}
(\T_{\beta\gamma} \V)(\C) -\beta\gamma\rho(\C) v(\C)
&\leq\sum_{m\in\M}\sum_{\ell\in\Ls}\C(m)P(\ell|m,w')\left(1+\beta\gamma\left(\V(\C'_{l,w,\C})-v(\C)\right)\right)
\\&\leq (\T_{\beta\gamma} (\V -v(\C)))(\C)\ ,
\end{aligned}
\end{equation*}
where $\rho(\C)$ is defined in Theorem \ref{thm:g}, $v(\cdot)$ is a function of $\C\in\SC$ and $w'$ is the chosen action by the DP operator in $(\T_{\beta\gamma} \V)(\C)$ and by Equation \eqref{for:ind:step1} it is obtained that
\begin{equation}\label{eq:ind:assum:b:1}
\begin{aligned}
(\T^{t}_{\TBb\gamma} \V)(\C) -\sum_{i=1}^{t-1}\left(\TBb\gamma\rho(\C)\right)^{i}(k-1)\overline{\theta}(\C)&\leq
\left(\T\TBb\left(\T^{t-1}_{\TBb\gamma} \V -\sum_{i=0}^{t-2}\left(\TBb\gamma\rho(\C)\right)^{i}(k-1)\overline{\theta}(\C)\right)\right)(\C)
\\&\triangleq\Upsilon(\C)\ .
\end{aligned}
\end{equation}
Furthermore, since we assume that Equation \eqref{eq:ind:assum} satisfies for $t-1$ and by the monotonicity of the operator $\T$, we have
\begin{equation}\label{eq:ind:assum:b:1_5}
\Upsilon(\C)\leq(\T\gT^{t-1}\V)(\C).
\end{equation}
Then, by Equation \eqref{eq:thm:g} we have
\begin{equation}\label{eq:ind:assum:b:2}
(\T \gT^{t-1}\V)(\C)
\leq\frac{1}{\TBb}(\gT \gT^{t-1}\V)(\C)+(k-1)\overline{\theta}(\C)\ .
\end{equation}
So, by Equations \eqref{eq:ind:assum:b:1} \eqref{eq:ind:assum:b:1_5}, and \eqref{eq:ind:assum:b:2} it is obtained that
\begin{equation*}
\TBb\left((\T^{t}_{\TBb\gamma} \V)(\C) -\sum_{i=0}^{t-1}\left(\TBb\gamma\rho(\C)\right)^{i}(k-1)\overline{\theta}(\C)\right)
\leq(\gT^{t}\V)(\C)\ .
\end{equation*}
So, it is obtained that Equation \eqref{eq:ind:assum} satisfies also for $t$. Therefore, by induction, Equation \eqref{eq:ind:assum} satisfies for every $t\geq 1$
\qed

\section{Bounding $\lambda$}\label{lem:bound:sum:supp}

In this section we provide Lemmas for boundedness of $\lambda$ from Section \ref{sec:action_space}.
\begin{lemma}\label{lem:bound:sum2}
For any $t\geq1$, under the zero initiation of the value function, namely, $\V(\C)=0$ for every $\C\in\SC$, it holds that
\begin{equation*}
\lambda = \frac{\Omega_{t,\C} }{  (\T^{t}_{\TBb\gamma}\V)(\C)} \leq \frac{(k-1) \bar{\theta}(\C)}{\rho(\C)}
\end{equation*}
\end{lemma}
\proof
By Lemma \ref{lem:bound:sum}, for a discount factor $\gamma'=\gamma\beta$ it follows that
\begin{equation*}
(\T^{t}_{\gamma\beta}\V)(\C)\geq\sum_{i=1}^{t}\left(\gamma\beta\right)^{i-1}\rho^{i}(\C)\ .
\end{equation*}
Therefore, by the definition of $\Omega_{t,\C}$, it is obtained that
\begin{equation*}
\frac{\Omega_{t,\C} }{  (\T^{t}_{\TBb\gamma}\V)(\C)} \leq \frac{\sum_{i=0}^{t-1}\left(\TBb\gamma\rho(\C)\right)^{i}(k-1)\overline{\theta}(\C)}{\sum_{i=1}^{t}\left(\gamma\beta\right)^{i-1}\rho^{i}(\C)}=\frac{(k-1) \bar{\theta}(\C)}{\rho(\C)}
\end{equation*}
So, Lemma \ref{lem:bound:sum2} is obtained.
\qed

In the following lemma we lower bound the value function. This lemma is used for the proof of Lemma \ref{lem:bound:sum2}.
\begin{lemma}\label{lem:bound:sum}
For any $t\geq1$, under the zero initiation of the value function, namely, $\V(\C)=0$ for every $\C\in\SC$, it holds that
\begin{equation}\label{eq:lem:bound:sum}
(\T^{t}\V)(\C)\geq\sum_{i=1}^{t}\gamma^{i-1}\rho^{i}(\C)\ ,
\end{equation}
where
\begin{equation*}
\rho(\C)=\max_{w\in\Ls^{\K}}\sum_{m\in\M}\sum_{\ell\in\Ls}\C(m)P(\ell|m,w)\ .
\end{equation*}
\end{lemma}
\proof
Let's denote $\widetilde{\T}$ as the DP operator under which the action $$\widetilde{\w}=\arg\max_{w\in\Ls^{\K}}\sum_{m\in\M}\sum_{\ell\in\Ls}\C(m)P(\ell|m,w)$$ is chosen at every state. We divide the proof into three parts.

\emph{First part:} By the monotonicity of the operator $\T$ and induction over $t$ it is obtained that
\begin{equation}\label{eq:lem:bound:sum:1}
(\T^{t}\V)(\C)\geq(\widetilde{\T}^{t}\V)(\C)\ .
\end{equation}
\emph{Second part:} Here we prove that for every $\C\in\SC$, it holds that
\begin{equation}\label{eq:to_proov}
(\widetilde{\T}^{t}\V)(\C)=\sum_{i=1}^{t}\gamma^{i-1}\sum_{m\in\M}\C(m)\left(\sum_{\ell\in\Ls}P(\ell|m,\widetilde{\w})\right)^{i}\ .
\end{equation}
We prove it by induction over $t$. Since
\begin{equation*}
(\widetilde{\T}\V)(\C)=\sum_{m\in\M}\sum_{\ell\in\Ls}\C(m)P(\ell|m,\widetilde{\w})\left(1+\gamma\V(\C'_{\ell,\widetilde{\w},\C})\right)\ ,
\end{equation*}
and the zero initiation, Equation \eqref{eq:to_proov} holds for $t=1$. Assume that Equation \eqref{eq:to_proov} holds for $t-1$. Recall Equation \eqref{eq:posterior_b} characterizing $\C'_{\ell,w,\C}(m')$
$$ \C'_{\ell,w,\C}(m')=\frac{\C(m')\p(\ell|m',w)}{\sum_{m\in\M}\C(m)\p(\ell|m,w)} \ .$$
Then,
\begin{equation*}
\begin{aligned}
(\widetilde{\T}^{t}\V)(\C)&=\sum_{m\in\M}\sum_{\ell\in\Ls}\C(m)P(\ell|m,\widetilde{\w})\left(1+\gamma(\widetilde{\T}^{t-1}\V)(\C'_{\ell,\widetilde{\w},\C})\right)
\\&=
\sum_{m\in\M}\sum_{\ell\in\Ls}\C(m)P(\ell|m,\widetilde{\w})\left(1+\gamma\sum_{i=1}^{t-1}\gamma^{i-1}\sum_{m'\in\M}\C'_{\ell,\widetilde{\w},\C}(m')\left(\sum_{\ell'\in\Ls}P(\ell'|m',\widetilde{\w})\right)^{i}\right)
\\&=
\sum_{m\in\M}\sum_{\ell\in\Ls}\C(m)P(\ell|m,\widetilde{\w})+\sum_{i=1}^{t-1}\gamma^{i}\sum_{m\in\M}\C(m)\left(\sum_{\ell'\in\Ls}P(\ell'|m',\widetilde{\w})\right)^{i+1}
\\&=\sum_{i=1}^{t}\gamma^{i-1}\sum_{m\in\M}\C(m)\left(\sum_{\ell'\in\Ls}P(\ell'|m',\widetilde{\w})\right)^{i}
\end{aligned}\ .
\end{equation*}
So, Equation \eqref{eq:to_proov} holds for any $t\geq1$.

\emph{Third part:}
By the convexity of $x^{i}$ for every natural $i$ and nonnegative $x$ it is obtained that
\begin{equation}\label{eq:to_proov_3}
\sum_{m\in\M}\C(m)\left(\sum_{\ell\in\Ls}P(\ell|m,\widetilde{\w})\right)^{i}\geq
\left(\sum_{m\in\M}\C(m)\sum_{\ell\in\Ls}P(\ell|m,\widetilde{\w})\right)^{i}=\rho^{i}(\C)\ .
\end{equation}
So by Equations \eqref{eq:lem:bound:sum:1}, \eqref{eq:to_proov} and \eqref{eq:to_proov_3} Lemma \ref{lem:bound:sum} is obtained.
\qed

\section{Proof of Theorem \ref{thm:g_d}}\label{thm:g_d:supp}
Here we provide the proof of Theorem \ref{thm:g_d}. The following definition generalizes the DP operator that include states which are not on the $\epsilon$-net.
\begin{definition}
	The DP operator $\cdT$ is an extension of the $\gdT$ operator for $\C\in\Delta_{\M}$ which are not on the $\eps$-net, $\SC$.
	\begin{equation}\label{eq:cdT:def}
	(\cdT\gdV^{t-1})(\C)=\max_{\w\in\gset}\sum_{m\in\M}\sum_{\ell\in\Ls}\C(m)\p(\ell|m,\w)(1+\gamma \gdV^{t-1}(\widehat{\C'_{\ell,\w,\C}}))\ .
	\end{equation}
	with $\gset$ being defined w.r.t.\ the finite state set $\SC$. Note that $\cdT$ and $\gdT$ are identical for $\C\in\SC$.
	
	Analogically, $\cdQ(\cdot)$ is an extension of $\gdQ(\cdot)$ for $\C\in\Delta_{\M}$ which are not on the $\eps$-net, $\SC$.
	\begin{equation*}
	\cdQ^{t}(\w,\C)= \sum_{m\in\M}\sum_{\ell\in\Ls}\C(m)\p(\ell|m,w)(1+\gamma \gdV^{t-1}(\widehat{\C'_{\ell,w,\C}})).
	\end{equation*}
	Note that $\cdQ(\cdot)$ and $\gdQ(\cdot)$ are identical for $\C\in\SC$.
\end{definition}

In the following lemma we bound the difference between the value function that is obtained by applying the DP-operator which is defined above to that obtained by the DP-operator $\gdT$, which is defined in Section \ref{sec:discrete}.

\begin{lemma}\label{def:def:desc}
	For zero initiation of the value function, it holds that,
	\begin{equation}\label{def:desc}
	\sup_{\C\in\Delta_{\M}}|\cdT\gdV^{t-1}(\C)-\gdT\gdV^{t-1}(\widehat{\C})|\leq \de .
	\end{equation}
	where $\frac{\eps}{B-\gamma}+\frac{2\eps\gamma}{(B-\gamma)^{2}}\leq\de$
\end{lemma}
\proof
We express $\cdT\gdV^{t-1}(\C)$ and $\gdT\gdV^{t-1}(\widehat{\C})$ as follows:
\begin{equation*}
\cdT\gdV^{t-1}(\C)=g(t,\C,w_{\C}^{t})
\end{equation*}
and
\begin{equation*}
\gdT\gdV^{t-1}(\widehat{\C})=g(t,\widehat{\C},w_{\widehat{\C}}^{t})\ ,
\end{equation*}
where for every $\C\in\Delta_{\M}$
\begin{equation*}
g(t,\C,\w_{\C}^{t})=\sum_{m\in\M}\sum_{\ell\in\Ls}\C(m)\p(\ell|m,\w)(1+\gamma g(t-1,\widehat{\C'},\w_{\widehat{\C'}}^{t-1}))\ ,
\end{equation*}
\begin{equation*}
g(0,\C,\w_{\C}^{0})=0\ ,
\end{equation*}
$\C'=\C'_{\ell,\w,\C}$ and $\w_{\C}^{t}$ stands for the set of actions which are taken at every states and iteration in the trajectory that begin at the state $\C$ and proceeds for $t$ iterations, under the operator $\cdT$ for the first iteration and then $\gdT$, (and only under the operator $\gdT$ for $\w_{\widehat{\C}}^{t-1}$).

By Assumption \ref{assum:sumP_B} it is easily obtained that
\begin{equation}\label{bound:g:func}
g(t,\C,\w_{\C}^{t})\leq \frac{1}{B-\gamma}\ ,
\end{equation}
for every $\C\in\Delta_{M}$ and $t\geq0$. 
Recall Equation \eqref{eq:posterior_b} characterizing $\C'_{\ell,w,\C}(m')$
$$ \C'_{\ell,w,\C}(m')=\frac{\C(m')\p(\ell|m',w)}{\sum_{m\in\M}\C(m)\p(\ell|m,w)}\ . $$ So, for the modification of $g(\cdot)$, which we denote as $\overline{g}(\cdot)$, where
\begin{equation*}
\overline{g}(t,\C,\w_{\C}^{t})=\sum_{m\in\M}\sum_{\ell\in\Ls}\C(m)\p(\ell|m,\w)(1+\gamma \overline{g}(t-1,\C',\w_{\C'}^{t-1}))\ ,
\end{equation*}
for any set $\w$ and $\C\in\Delta_{M}$, it is obtained that
\begin{equation}\label{bound:g:over_g:func:state}
|g(t,\C,\w)- \overline{g}(t,\C,\w)|\leq\frac{\eps\gamma}{(B-\gamma)^{2}}\ .
\end{equation}

In addition, by plugging Equation \eqref{eq:posterior_b} in the recursion of $\overline{g}(\cdot)$, it is obtained that $\overline{g}(\cdot)$ is linear in $\C$. So, for every two states $\C^{1}\in\Delta_{M}$ and $\C^{2}\in\Delta_{M}$, such that $|\C^{1}-\C^{2}|_{1}\leq\eps$ and a set of actions $\w$, it is obtained by the linearity of $\overline{g}(\cdot)$ and Equations \eqref{bound:g:func} that
\begin{equation}\label{bound:over:g:func:state}
|\overline{g}(t,\C^{1},\w)- \overline{g}(t,\C^{2},\w)|\leq\frac{\eps}{B-\gamma}\ .
\end{equation}

So, by Equations \eqref{bound:g:over_g:func:state} and \eqref{bound:over:g:func:state}, it is obtained that
\begin{equation}\label{bound:g:func:state}
|g(t,\C^{1},\w)- g(t,\C^{2},\w)|\leq\frac{\eps}{B-\gamma}+\frac{2\eps\gamma}{(B-\gamma)^{2}}\ .
\end{equation}

In addition, by the definitions of the $\gdT$ and the $\cdT$ operators we have that
\begin{equation}\label{bound:g:func:action:set}
g(t,\C,\w_{\C}^{t})\geq g(t,\C,\w)\ ,
\end{equation}
for every state $\C$ and set $\w$. Therefore, since $\C$ and $\widehat{\C}$ satisfies that $|\C-\widehat{\C}|_{1}\leq\eps$ and by Equations \eqref{bound:g:func:state} and \eqref{bound:g:func:action:set} it is obtained that
\begin{equation*}
g(t,\C,\w_{\C}^{t})\geq g(t,\C,\w_{\widehat{\C}}^{t})\geq g(t,\widehat{\C},\w_{\widehat{\C}}^{t})-\frac{\eps}{B-\gamma}-\frac{2\eps\gamma}{(B-\gamma)^{2}}\ ,
\end{equation*}
and that
\begin{equation*}
g(t,\widehat{\C},\w_{\widehat{\C}}^{t})\geq g(t,\widehat{\C},\w_{\C}^{t})\geq g(t,\C,\w_{\C}^{t})-\frac{\eps}{B-\gamma}-\frac{2\eps\gamma}{(B-\gamma)^{2}}\ .
\end{equation*}
So,
\begin{equation}
|g(t,\C,\w_{\C}^{t})-g(t,\widehat{\C},\w_{\widehat{\C}}^{t})|\leq \frac{\eps}{B-\gamma}+\frac{2\eps\gamma}{(B-\gamma)^{2}}\ .
\end{equation}
\qed

In the following lemma we upper bound the value function that is obtained by the $\gdT$ operator.

\begin{lemma}\label{Bound_gv:d}
	For every $\C\in\Delta_{M}$, $t\geq0$ and zero initiation of the value function (namely, $\gdV^{0}=0$), it holds that $$\gdV^{t}(\widehat{\C})\leq\frac{1}{B-\gamma}\ .$$
\end{lemma}
\proof
Similar to the proof of Lemma \ref{Bound_gv} in section \ref{thm:g:supp} in the supplementary material.
\qed

In the following two lemmas we show a monotonic property of the value function that is obtained by the $\gdT$ operator..

\begin{lemma} \label{lem:new:AV>BV_G base:d}
	Let $\C^{1},\C^{2}\in\Delta_{M}$ and let $A_{1}$ and $B_{2}$ be a pair of positive constants. Assume that 
	\begin{equation}
	A_{1}\C^{1}(m)\geq B_{2}\C^{2}(m)\ ,
	\end{equation}
	for all $m \in \M$. Then it holds that 
	$$A_{1}\cdT\gdV^{0}(\C^{1})\geq B_{2}\cdT\gdV^{0}(\C^{2})\ .$$
\end{lemma}
\begin{proof}
	True for the initiate value function $\gdV^{0}(\C)=0,\ \forall\C\in\SC$.
	\qed
\end{proof}
\begin{lemma}\label{lem:new:AV>BV_G:d}
	Let $\C^{1},\C^{2}\in\Delta_{M}$ and let $A_{1}$ and $B_{2}$ be a pair of positive constants. Assume that 
	\begin{equation}\label{lem:new:AV>BV:cond_G:d}
	A_{1}\C^{1}(m)\geq B_{2}\C^{2}(m)\ ,
	\end{equation}
	for all $m \in \M$. We have that for any integer $t\geq0$ it holds that
	$$A_{1}\cdT\gdV^{t}(\C^{1})\geq B_{2}\cdT\gdV^{t}(\C^{2})-\de\left(A_{1}+B_{2}\right)\sum_{i=1}^{t+1}\left(\frac{\gamma}{B}\right)^{i}\ .$$
\end{lemma}
\begin{proof}
	We prove the claim by induction over $t$. The base case for $t=0$ holds due to Lemma~\ref{lem:new:AV>BV_G base:d}. Assume that the lemma is satisfied for $t-1$.
	Recall Equation \eqref{eq:posterior_b} characterizing $\C'_{\ell,w,\C}(m')$
	$$ \C'_{\ell,w,\C}(m')=\frac{\C(m')\p(\ell|m',w)}{\sum_{m\in\M}\C(m)\p(\ell|m,w)}\ . $$
	By plugging in with Equation \eqref{lem:new:AV>BV:cond_G:d} we get that
	$$A_{1}\left(\sum_{m\in\M}\C^{1}(m)\p(\ell|m,w)\right)\C'_{\ell,w,\C^{1}}(m')\geq B_{2}\left(\sum_{m\in\M}\C^{2}(m)\p(\ell|m,w)\right)\C'_{\ell,w,\C^{2}}(m'),$$
	for any $\w\in\Ls^{\K}$, $\ell\in\Ls$ and $m'\in\M$, as $\p(\ell|m',w) \geq 0$. Therefore, by the induction assumption applied for 
	$$A_1'(\ell) = A_{1}\left(\sum_{m\in\M}\C^{1}(m)\p(\ell|m,w)\right) \ \ \ , \ \ \ B_2'(\ell) = B_{2}\left(\sum_{m\in\M}\C^{2}(m)\p(\ell|m,w)\right) \ ,$$
	\begin{equation*}
	\begin{aligned}
	A_{1}\sum_{m\in\M}\C^{1}(m)\p(\ell|m,w)\cdT\gdV^{t-1}\left(\C'_{\ell,w,\C^{1}}\right)&\geq
	B_{2}\sum_{m\in\M}\C^{2}(m)\p(\ell|m,w)\cdT\gdV^{t-1}\left(\C'_{\ell,w,\C^{2}}\right)
	\\&-\de\left(A_{1}'(\ell)+B_{2}'(\ell)\right)\sum_{i=1}^{t}\left(\frac{\gamma}{B}\right)^{i}\ ,
	\end{aligned}
	\end{equation*}
	for every $\ell\in\Ls$ and $\w\in\Ls^{\K}$. So, by Equation \eqref{def:desc}
	\begin{equation}\label{lem:new:AV>BV_G:1:d}
	\begin{aligned}
	A_{1}\sum_{m\in\M}\C^{1}(m)\p(\ell|m,w)\left(\gdV^{t}\left(\widehat{\C'_{\ell,w,\C^{1}}}\right)+\de\right)\geq
	&B_{2}\sum_{m\in\M}\C^{2}(m)\p(\ell|m,w)\left(\gdV^{t}\left(\widehat{\C'_{\ell,w,\C^{2}}}\right)-\de\right)
	\\&-\de\left(A_{1}'(\ell)+B_{2}'(\ell)\right)\sum_{i=1}^{t}\left(\frac{\gamma}{B}\right)^{i}\ .
	\end{aligned}
	\end{equation}
	
	Furthermore, by Equation \eqref{lem:new:AV>BV:cond_G:d}
	\begin{equation}\label{lem:new:AV>BV_G:2:d}
	A_{1}\sum_{m\in\M}\C^{1}(m)\p(\ell|m,w)\geq
	B_{2}\sum_{m\in\M}\C^{2}(m)\p(\ell|m,w)\ ,
	\end{equation}	
	for every $\ell\in\Ls$ and $\w\in\Ls^{\K}$. So, by the fact that
	$$
	A_{1}\cdQ^{t+1}(w,\C^{1})=A_{1}\sum_{m\in\M}\sum_{\ell\in\Ls}\C^{1}(m)\p(\ell|m,w)\left(1+\gamma\gdV^{t}\left(\widehat{\C'_{\ell,w,\C^{1}}}\right)\right)\ ,
	$$
	and also respectively for $B_{2}$ and $\C^{2}$,
	it is obtained by Equations \eqref{lem:new:AV>BV_G:1:d} and \eqref{lem:new:AV>BV_G:2:d} and Assumption \ref{assum:sumP_B} that
	\begin{equation}\label{lem:new:AV>BV:Q:eq_G:finall:d}
	A_{1}\cdQ^{t+1}(w,\C^{1})\geq B_{2}\cdQ^{t+1}(w,\C^{2})-\de\left(A_{1}+B_{2}\right)\sum_{i=1}^{t+1}\left(\frac{\gamma}{B}\right)^{i}\ ,
	\end{equation}
	for any $w\in\Ls^{\K}$.
	So, by the definition of the $\cdT$ operator the result is obtained.
	\qed
\end{proof}

In the following two lemmas we show an "almost"-convexity property of the value function that is obtained by the $\gdT$ operator.

\begin{lemma} \label{lem:new:2:AV>BV_G base:d}
	Let $\C,\C^{1},\C^{2}\in\SC$ and let $A$, $B_{1}$ and $B_{2}$ be a tuple of positive constants.  Assume that 
	\begin{equation}
	A\C(m)= B_{1}\C^{1}(m)+B_{2}\C^{2}(m)\ ,
	\end{equation}
	for all $m \in \M$. Then it holds that
	$$A\cdT\gdV^{0}(\C)\leq B_{1}\cdT\gdV^{0}(\C^{1})+B_{2}\cdT\gdV^{0}(\C^{2})\ .$$
\end{lemma}
\begin{proof}
	True for the initiate value function $\gdV^{0}(\C)=0,\ \forall\C\in\SC$.
	\qed
\end{proof}

\begin{lemma}[$\epsilon$-Convexity]\label{lem:new:2:AV>BV_G:d}
	Let $\C,\C^{1},\C^{2}\in\Delta_{M}$ and let $A$, $B_{1}$ and $B_{2}$ be a tuple of positive constants. Assume that 
	\begin{equation}\label{lem:new:2:AV>BV:cond_G:d}
	A\C(m)= B_{1}\C^{1}(m)+B_{2}\C^{2}(m)\ ,
	\end{equation}
	for all $m \in \M$. We have that for any integer $t\geq0$ it holds that
	$$A\gdV^{t}(\widehat{\C})\leq B_{1}\gdV^{t}(\widehat{\C^{1}})+B_{2}\gdV^{t}(\widehat{\C^{2}})+\de\left(A_1+B_1+B_2\right)\sum_{i=0}^{t}\left(\frac{\gamma}{B}\right)^{i}\ .$$
\end{lemma}

\begin{proof}
	True for $t=0$ by the zero initiation. For $t\geq1$ we first prove that
	\begin{equation}\label{to_prove_ind}
	A\cdT\gdV^{t}(\C)\leq B_{1}\cdT\gdV^{t}(\C^{1})+B_{2}\cdT\gdV^{t}(\C^{2})+\de\left(A_1+B_1+B_2\right)\sum_{i=1}^{t+1}\left(\frac{\gamma}{B}\right)^{i}\ .
	\end{equation}
	We prove the claim (Equation \ref{to_prove_ind}) by induction over $t$. The base case for $t=0$ holds due to Lemma~\ref{lem:new:2:AV>BV_G base:d}. Assume that Equation \ref{to_prove_ind} is satisfied for $t-1$.
	Recall Equation \eqref{eq:posterior_b} characterizing $\C'_{\ell,w,\C}(m')$
	$$ \C'_{\ell,w,\C}(m')=\frac{\C(m')\p(\ell|m',w)}{\sum_{m\in\M}\C(m)\p(\ell|m,w)}\ . $$
	By plugging in with Equation \eqref{lem:new:2:AV>BV:cond_G:d} we get that
	\begin{equation*}
	\begin{aligned}
	&A\left(\sum_{m\in\M}\C(m)\p(\ell|m,w)\right)\C'_{\ell,w,\C}(m')=
	\\& B_{1}\left(\sum_{m\in\M}\C^{1}(m)\p(\ell|m,w)\right)\C'_{\ell,w,\C^{1}}(m')+B_{2}\left(\sum_{m\in\M}\C^{2}(m)\p(\ell|m,w)\right)\C'_{\ell,w,\C^{2}}(m')
	\end{aligned}\ ,
	\end{equation*}
	for any $\w\in\Ls^{\K}$, $\ell\in\Ls$ and $m'\in\M$, as $\p(\ell|m',w) \geq 0$. Therefore, by the induction assumption applied for 
	\begin{align*}
	&A'(\ell) = A\left(\sum_{m\in\M}\C(m)\p(\ell|m,w)\right) \ \ \ \ \ , \ \ \ B_1'(\ell) = B_{1}\left(\sum_{m\in\M}\C^{1}(m)\p(\ell|m,w)\right) \ \ \ , \ \ \ \\&B_2'(\ell) = B_{2}\left(\sum_{m\in\M}\C^{2}(m)\p(\ell|m,w)\right) \ ,
	\end{align*}
	\begin{equation*}
	\begin{aligned}
	A\sum_{m\in\M}\C(m)\p(\ell|m,w)\cdT\gdV^{t-1}\left(\C'_{\ell,w,\C}\right)\leq
	&B_{1}\sum_{m\in\M}\C^{1}(m)\p(\ell|m,w)\cdT\gdV^{t-1}\left(\C'_{\ell,w,\C^{1}}\right)\\&+
	B_{2}\sum_{m\in\M}\C^{2}(m)\p(\ell|m,w)\cdT\gdV^{t-1}\left(\C'_{\ell,w,\C^{2}}\right)\\&+\de\left(A_1'(\ell)+B_1'(\ell)+B_2'(\ell)\right)\sum_{i=1}^{t}\left(\frac{\gamma}{B}\right)^{i}
	\end{aligned}\ ,
	\end{equation*}
	for every $\ell\in\Ls$ and $\w\in\Ls^{\K}$. So, by Equation \eqref{def:desc}
	\begin{equation}\label{lem:new:2:AV>BV_G:1:d}
	\begin{aligned}
	A\sum_{m\in\M}\C(m)\p(\ell|m,w)\left(\gdV^{t}\left(\widehat{\C'_{\ell,w,\C}}\right)-\de\right)\leq
	&B_{1}\sum_{m\in\M}\C^{1}(m)\p(\ell|m,w)\left(\gdV^{t}\left(\widehat{\C'_{\ell,w,\C^{1}}}\right)+\de\right)\\&+
	B_{2}\sum_{m\in\M}\C^{2}(m)\p(\ell|m,w)\left(\gdV^{t}\left(\widehat{\C'_{\ell,w,\C^{2}}}\right)+\de\right)
	\\&+\de\left(A_1'(\ell)+B_1'(\ell)+B_2'(\ell)\right)\sum_{i=1}^{t}\left(\frac{\gamma}{B}\right)^{i}
	\end{aligned}\ .
	\end{equation}
	Furthermore, by Equation \eqref{lem:new:2:AV>BV:cond_G:d}
	\begin{equation}\label{lem:new:2:AV>BV_G:2:d}
	A\sum_{m\in\M}\C(m)\p(\ell|m,w)=B_{1}\sum_{m\in\M}\C^{1}(m)\p(\ell|m,w)+
	B_{2}\sum_{m\in\M}\C^{2}(m)\p(\ell|m,w)\ ,
	\end{equation}	
	for every $\ell\in\Ls$ and $\w\in\Ls^{\K}$. So, by the fact that
	$$
	A\cdQ^{t+1}(w,\C)=A\sum_{m\in\M}\sum_{\ell\in\Ls}\C(m)\p(\ell|m,w)\left(1+\gamma\gdV^{t}\left(\widehat{\C'_{\ell,w,\C}}\right)\right)\ ,
	$$
	and also respectively for $B_{1}$, $\C^{1}$, $B_{2}$ and $\C^{2}$,
	it is obtained by Equations \eqref{lem:new:2:AV>BV_G:1:d} and \eqref{lem:new:2:AV>BV_G:2:d} and Assumption \ref{assum:sumP_B} that
	\begin{equation}\label{lem:new:2:AV>BV:Q:eq_G:finall:d}
	A\cdQ^{t+1}(w,\C)\leq B_{1}\cdQ^{t+1}(w,\C^{1})+B_{2}\cdQ^{t+1}(w,\C^{2})+\de\left(A_1+B_1+B_2\right)\sum_{i=1}^{t+1}\left(\frac{\gamma}{B}\right)^{i}\ ,
	\end{equation}
	for any $w\in\Ls^{\K}$.
	So, by the definition of the $\cdT$ operator it is obtained that
	$$A\cdT\gdV^{t}(\C)\leq B_{1}\cdT\gdV^{t}(\C^{1})+B_{2}\cdT\gdV^{t}(\C^{2})+\de\left(A_1+B_1+B_2\right)\sum_{i=1}^{t+1}\left(\frac{\gamma}{B}\right)^{i}\ .$$
	So, Equation \eqref{to_prove_ind} holds for any $t$.
	Therefore by Equation \eqref{def:desc} it is obtained that
	$$A\gdV^{t+1}(\widehat{\C})\leq B_{1}\gdV^{t+1}(\widehat{\C^{1}})+B_{2}\gdV^{t+1}(\widehat{\C^{2}})+\de\left(A_1+B_1+B_2\right)\sum_{i=0}^{t+1}\left(\frac{\gamma}{B}\right)^{i}\ .$$
	\qed
\end{proof}

In the following lemma we show that the Q-function, obtained by the $\gdT$ operator is "almost"-submodular.

\begin{lemma}[Almost-Submodularity]\label{lem:new:2:submodularity:d}
	For any $\C\in\SC$, integer $t\geq1$, where $\w_{b}\supset\w_{a}$ and $\ell'\not\in w_{b}$ it holds that
	\begin{equation*}
	\gdQ^{t}(\{w_{a}\cup \ell'\},\C)-\gdQ^{t}(w_{a},\C)\geq\gdQ^{t}(\{w_{b}\cup \ell'\},\C)-\gdQ^{t}(w_{b},\C) -\theta_{d}(\ell',w_{a},\C)\ ,
	\end{equation*}
	where $$\theta_{d}(\ell',w_{a},\C)=\sum_{m\in\M}\C(m)P(\ell'|m,w_{a}\cup \ell')\sum_{\ell\in\Ls}P(\ell|m,w_{a})\gamma\frac{1}{B-\gamma}+\frac{5\de B}{B-\gamma}\ .$$
\end{lemma}
\proof
Let
\begin{equation}
\gdQ^{t}(\{w_{a}\cup \ell'\},\C)-\gdQ^{t}(w_{a},\C)-\left(\gdQ^{t}(\{w_{b}\cup \ell'\},\C)-\gdQ^{t}(w_{b},\C)\right)=\Phi_{1,d}+\Phi_{2,d}^{1}+\Phi_{2,d}^{2}\ ,
\end{equation}
where
\begin{align*}
&\begin{aligned}
\Phi_{1,d}\triangleq&\sum_{m\in\M}\C(m)\sum_{\ell\in\Ls}P(\ell|m,w_{a}\cup \ell')-\sum_{m\in\M}\C(m)\sum_{\ell\in\Ls}P(\ell|m,w_{a})
\\&+\sum_{m\in\M}\C(m)\sum_{\ell\in\Ls}P(\ell|m,w_{b})-\sum_{m\in\M}\C(m)\sum_{\ell\in\Ls}P(\ell|m,w_{b}\cup \ell')
\end{aligned}\ ,\\
&\begin{aligned}
\Phi_{2,d}^{1}\triangleq&\sum_{m\in\M}\C(m)P(\ell'|m,w_{a}\cup \ell')\gamma\gdV^{t-1}(\widehat{\C'_{\ell',\{w_{a}\cup \ell'\},\C}})
\\&+\sum_{m\in\M}\C(m)\sum_{\ell\in\Ls}P(\ell|m,w_{b})\gamma\gdV^{t-1}(\widehat{\C'_{\ell,w_{b},\C}})
\\&-\sum_{m\in\M}\C(m)\sum_{\ell\in\Ls}P(\ell|m,w_{b}\cup \ell')\gamma\gdV^{t-1}(\widehat{\C'_{\ell,\{w_{b}\cup \ell'\},\C}})
\end{aligned}\ ,\\
&\text{and} \nonumber\\
&\begin{aligned}
\Phi_{2,d}^{2}\triangleq&\sum_{m\in\M}\C(m)\sum_{\ell\in\Ls\setminus \ell'}P(\ell|m,w_{a}\cup \ell')\gamma\gdV^{t-1}(\widehat{\C'_{\ell,\{w_{a}\cup \ell'\},\C}})
\\&-\sum_{m\in\M}\C(m)\sum_{\ell\in\Ls}P(\ell|m,w_{a})\gamma\gdV^{t-1}(\widehat{\C'_{\ell,w_{a},\C}})
\end{aligned}\ .
\end{align*}
Then by Proposition \ref{assum1_G} it is obtained that 
\begin{equation}\label{eq:lem:new:2:submodularity:Phi1_1:bound:d}
\Phi_{1,d}\geq 0\ .
\end{equation}
So,  for $t=1$, by the zero initiation, $\Phi_{2,d}^{1}=\Phi_{2,d}^{2}=0$, and therefore the Lemma holds. So, in the remain of this proof we consider the case of $t\geq2$.

For bounding $\Phi_{2,d}^{1}$ we note that according to the definition of $\C'_{\ell,w,\C}$ and Assumption \ref{assum:w_w'}, it is obtained for every $\ell\in\w_{b}$ that,
\begin{equation*}
\sum_{m\in\M}\C(m)P(\ell|m,w_{b})\gamma\C'_{\ell,w_{b},\C}(m')
\geq\sum_{m\in\M}\C(m)P(\ell|m,w_{b}\cup \ell')\gamma\C'_{\ell,\{w_{b}\cup \ell'\},\C}(m')\ .
\end{equation*}
and for $\ell'$ that,
\begin{equation*}
\sum_{m\in\M}\C(m)P(\ell'|m,w_{a}\cup \ell')\C'_{\ell',\{w_{a}\cup \ell'\},\C}(m')
\geq\sum_{m\in\M}\C(m)P(\ell'|m,w_{b}\cup \ell')\C'_{\ell',\{w_{b}\cup \ell'\},\C}(m')\ ,
\end{equation*}
for every $m'\in\M$.
Therefore, by Lemma \ref{lem:new:AV>BV_G:d}, for every $\ell\in\w_{b}$ it is obtained that
\begin{equation*}
\sum_{m\in\M}\C(m)P(\ell|m,w_{b})\gamma\cdT\gdV^{t-2}(\C'_{\ell,w_{b},\C})
\geq\sum_{m\in\M}\C(m)P(\ell|m,w_{b}\cup \ell')\gamma\cdT\gdV^{t-2}(\C'_{\ell,\{w_{b}\cup \ell'\},\C})-\delta_{1}(l)\ ,
\end{equation*}
and for $\ell'$ it is obtained that
\begin{equation*}
\begin{aligned}
&\sum_{m\in\M}\C(m)P(\ell'|m,w_{a}\cup \ell')\gamma\cdT\gdV^{t-2}(\C'_{\ell',\{w_{a}\cup \ell'\},\C})
\\&\geq\sum_{m\in\M}\C(m)P(\ell'|m,w_{b}\cup \ell')\gamma\cdT\gdV^{t-2}(\C'_{\ell',\{w_{b}\cup \ell'\},\C})-\delta_{1}(l')\ .
\end{aligned}
\end{equation*}
where $$\delta_{1}(l)=\de\left(\sum_{m\in\M}\C(m)P(\ell|m,w_{b})+\sum_{m\in\M}\C(m)P(\ell|m,w_{b}\cup \ell')\right)\sum_{i=1}^{t-1}\left(\frac{\gamma}{B}\right)^{i}\ ,$$ and 
$$\delta_{2}(l')=\de\left(\sum_{m\in\M}\C(m)P(\ell'|m,w_{a}\cup \ell')+\sum_{m\in\M}\C(m)P(\ell'|m,w_{b}\cup \ell')\right)\sum_{i=1}^{t-1}\left(\frac{\gamma}{B}\right)^{i}\ ,$$

So, by Equation \eqref{def:desc} and Assumption \ref{assum:sumP_B},
\begin{equation}\label{eq:lem:new:2:submodularity:Phi1_2:bound:d}
\Phi_{2,d}^{1}\geq 3\de\sum_{i=0}^{t-1}\left(\frac{\gamma}{B}\right)^{i}\ .
\end{equation}
In addition, we note that for every $\ell\in\w_{a}$ it is obtained by Assumption \ref{assum:w_w'} that
\begin{equation*}
\begin{aligned}
\sum_{m\in\M}\C(m)P(\ell|m,w_{a})\gamma\C'_{\ell,w_{a},\C}=&\sum_{m\in\M}\C(m)P(\ell|m,w_{a}\cup \ell')\gamma\C'_{\ell,\{w_{a}\cup \ell'\},\C}
\\&+\sum_{m\in\M}\C(m)P(\ell'|m,w_{a}\cup \ell')P(\ell|m,w_{a})\gamma\tilde{\C}\ ,
\end{aligned}
\end{equation*}
where
\begin{equation*}
\tilde{\C}(m')=\frac{\C(m')P(\ell'|m',w_{a}\cup \ell')P(\ell|m',w_{a})}{\sum_{m\in\M}\C(m)P(\ell'|m,w_{a}\cup \ell')P(\ell|m,w_{a})}\ .
\end{equation*}
So, by Lemmas \ref{lem:new:2:AV>BV_G:d} and \ref{Bound_gv:d},
\begin{equation}\label{eq:lem:new:2:submodularity:Phi2_2:bound:d}
\Phi_{2,d}^{2}\geq -\sum_{m\in\M}\C(m)P(\ell'|m,w_{a}\cup \ell')\sum_{\ell\in\Ls}P(\ell|m,w_{a})\gamma\frac{1}{B-\gamma}-2\de\sum_{i=0}^{t-1}\left(\frac{\gamma}{B}\right)^{i}\ .
\end{equation}
Therefore, by Equations \eqref{eq:lem:new:2:submodularity:Phi1_1:bound:d}, \eqref{eq:lem:new:2:submodularity:Phi1_2:bound:d} and \eqref{eq:lem:new:2:submodularity:Phi2_2:bound:d} it is obtained that
\begin{equation*}
\gdQ^{t}(\{w_{a}\cup \ell'\},\C)-\gdQ^{t}(w_{a},\C)\geq\gdQ^{t}(\{w_{b}\cup \ell'\},\C)-\gdQ^{t}(w_{b},\C) -\theta_{d}(\ell',w_{a},\C)\ ,
\end{equation*}
where 
\begin{equation*}
\begin{aligned}
\theta_{d}(\ell',w_{a},\C)= &\sum_{m\in\M}\C(m)P(\ell'|m,w_{a}\cup \ell')\sum_{\ell\in\Ls}P(\ell|m,w_{a})\gamma\frac{1}{B-\gamma}+5\de\sum_{i=0}^{t-1}\left(\frac{\gamma}{B}\right)^{i}
\\&\leq\sum_{m\in\M}\C(m)P(\ell'|m,w_{a}\cup \ell')\sum_{\ell\in\Ls}P(\ell|m,w_{a})\gamma\frac{1}{B-\gamma}+\frac{5\de B}{B-\gamma}
\end{aligned}\ .
\end{equation*}c
\qed

In the following lemma we show that the Q-function, obtained by the $\gdT$ operator is "almost"-monotone.

\begin{lemma}[Almost-Monotonicity]\label{lem:new:2:monotonicity:d}
	If $B\geq1+\gamma$, then for any $\C\in\SC$, a set of content items $\w$ such that $\ell'\not\in\w$ and $t\geq0$ it holds that $\gdQ^{t}(\{\w\cup \ell'\},\C)\geq\gdQ^{t}(\w,\C)-\frac{2\gamma\de}{B-\gamma}$. Where $B$ is the constant in Assumption \ref{assum:sumP_B}.
\end{lemma}
\proof
\begin{equation*}
\gdQ^{t}(\{\w\cup \ell'\},\C)-\gdQ^{t}(\w,\C)=\Psi_{1,d}+\Psi_{2,d}\ ,
\end{equation*}
where
\begin{equation*}
\Psi_{1,d}\triangleq\sum_{m\in\M}\C(m)\sum_{\ell\in\Ls}P(\ell|m,w\cup \ell')-\sum_{m\in\M}\C(m)\sum_{\ell\in\Ls}P(\ell|m,w)\ ,
\end{equation*}
and
\begin{equation*}
\Psi_{2,d}\triangleq\sum_{m\in\M}\C(m)\sum_{\ell\in\Ls}P(\ell|m,w\cup \ell')\gamma\gdV^{t-1}(\widehat{\C'_{\ell,\{w\cup \ell'\},\C}})-\sum_{m\in\M}\C(m)\sum_{\ell\in\Ls}P(\ell|m,w)\gamma\gdV^{t-1}(\widehat{\C'_{\ell,w,\C}})\ .
\end{equation*}
Then, by Assumption \ref{assum:w_w'} it is obtained that
\begin{equation}\label{eq:lemm:new:2:Phi1:bound:d}
\Psi_{1,d}=\sum_{m\in\M}\C(m)P(\ell'|m,w\cup \ell')\left(1-\sum_{\ell\in\Ls}P(\ell|m,w)\right)\ .
\end{equation}
For bounding $\Psi_{2,d}$, recall Equation \eqref{eq:posterior_b} characterizing $\C'_{\ell,w,\C}(m')$
$$ \C'_{\ell,w,\C}(m')=\frac{\C(m')\p(\ell|m',w)}{\sum_{m\in\M}\C(m)\p(\ell|m,w)}\ . $$
Then, for every $\ell\in\w$ it is obtained by Assumption \ref{assum:w_w'} that
\begin{equation*}
\begin{aligned}
\sum_{m\in\M}\C(m)P(\ell|m,w)\gamma\C'_{\ell,w,\C}=&\sum_{m\in\M}\C(m)P(\ell|m,w\cup \ell')\gamma\C'_{\ell,\{w\cup \ell'\},\C}
\\&+\sum_{m\in\M}\C(m)P(\ell'|m,w\cup \ell')P(\ell|m,w)\gamma\tilde{\C}\ ,
\end{aligned}
\end{equation*}
where
\begin{equation*}
\tilde{\C}(m')=\frac{\C(m')P(\ell'|m',w\cup \ell')P(\ell|m',w)}{\sum_{m\in\M}\C(m)P(\ell'|m,w\cup \ell')P(\ell|m,w)}\ .
\end{equation*}
So, by Lemmas \ref{Bound_gv:d} and \ref{lem:new:2:AV>BV_G:d},
\begin{equation}\label{eq:lem:new:2:mono:d}
\begin{aligned}
&\sum_{m\in\M}\C(m)P(\ell|m,w)\gamma\gdV^{t-1}\left(\widehat{\C'_{\ell,w,\C}}\right)\leq
\\
&\sum_{m\in\M}\C(m)P(\ell|m,w\cup \ell')\gamma\gdV^{t-1}\left(\widehat{\C'_{\ell,\{w\cup \ell'\},\C}}\right)
+\sum_{m\in\M}\C(m)P(\ell'|m,w\cup \ell')P(\ell|m,w)\gamma\frac{1}{B-\gamma}
\\&+\de\gamma\left(P(\ell|m,w)+P(\ell|m,w\cup \ell')+P(\ell'|m,w\cup \ell')P(\ell|m,w)\right)\sum_{i=0}^{t-1}\left(\frac{\gamma}{B}\right)^{i}
\end{aligned}\ .
\end{equation}
Therefore, by Equation \eqref{eq:lem:new:2:mono:d} it is obtained that
\begin{equation}\label{eq:lemm:new:2:Phi2:bound:d}
\begin{aligned}
\Psi_{2,d}\geq&-\sum_{m\in\M}\C(m)\sum_{\ell\in\Ls}P(\ell'|m,w\cup \ell')P(\ell|m,w)\gamma\frac{1}{B-\gamma}
\\&-\de\gamma\sum_{\ell\in\Ls}\left(P(\ell|m,w)+P(\ell|m,w\cup \ell')+P(\ell'|m,w\cup \ell')P(\ell|m,w)\right)\sum_{i=0}^{t-1}\left(\frac{\gamma}{B}\right)^{i}
\end{aligned}\ .
\end{equation}
So, by Equations \eqref{eq:lemm:new:2:Phi1:bound:d} and \eqref{eq:lemm:new:2:Phi2:bound:d},
\begin{equation*}
\begin{aligned}
\Psi_{1,d}+\Psi_{2,d}\geq&\sum_{m\in\M}\C(m)P(\ell'|m,w\cup \ell')\left(1-\frac{B}{B-\gamma}\sum_{\ell\in\Ls}P(\ell|m,w)\right)
\\&-\de\gamma\sum_{\ell\in\Ls}\left(P(\ell|m,w)+P(\ell|m,w\cup \ell')+P(\ell'|m,w\cup \ell')P(\ell|m,w)\right)\sum_{i=0}^{t-1}\left(\frac{\gamma}{B}\right)^{i}
\end{aligned}\ .
\end{equation*}
Then, by Assumptions \ref{assum:sumP_B} and \ref{assum:w_w'}, it is obtained that $\Psi_{1}+\Psi_{2}\geq -\frac{2\gamma\de}{B-\gamma}$ for $B\geq1+\gamma$, and therefore the Lemma holds.
\qed

\noindent{\bf Proof of Theorem \ref{thm:g_d}:\ }

\proof
For proving Theorem \ref{thm:g_d} it is sufficient to show that for $\de$ for which

$$\frac{\eps}{B-\gamma}+\frac{2\eps\gamma}{(B-\gamma)^{2}}\leq\de \ ,$$
it is obtained that
\begin{equation}\label{thm:greedy:first_d}
\left(\gdT^{t}\V\right)(\C)\leq (\T^{t}\V)(\C)+\de\sum_{i=0}^{t-1}\left(\frac{\gamma}{B}\right)^{i}
\end{equation}
and that
\begin{equation}\label{eq:ind:assum_d}
\TBb\left((\T^{t}_{\TBb\gamma} \V)(\C) -\Omega^{d}_{t,\C,\de}\right) \leq(\gdT^{t}\V)(\C)
\end{equation}
where $\TBb\geq 0.63$ is defined in Equation \eqref{eq:TBb}, 
$$\Omega^{d}_{t,\C,\de}=\sum_{i=0}^{t-1}\left(\TBb\gamma\rho(\C)\right)^{i}\left(\frac{\de}{\TBb}+(k-1)\overline{\theta}_{d}(\C)\right),\quad\rho(\C)\triangleq\max_{\w\in\Ls^{\K}}\sum_{m\in\M}\C(m)\sum_{\ell\in\Ls}P(\ell|m,w),$$
and
\begin{equation}\label{THM:Theta_d}
\begin{aligned}
&\overline{\theta}_{d}(\C)\triangleq \frac{5B\de}{B-\gamma}+\frac{2\gamma k\de}{(k-1)(B-\gamma)} + \\
&\max_{\ell'\in\Ls,\w\in\Ls^{\K}}\sum_{m\in\M}\C(m)P(\ell'|m,w\cup \ell')\sum_{\ell\in\Ls}P(\ell|m,w)\gamma\frac{1}{B-\gamma}
\end{aligned}
\end{equation}

We prove Equation \eqref{thm:greedy:first_d} by induction over $t$. Since the value obtained by the maximization in Equation \eqref{eq:cdT:def} is equal or smaller than the accurate maximal value, by Lemma \ref{def:def:desc} and by the zero initiation, we have that
\begin{equation}\label{eq:thm:g:1_d}
(\gdT \V)(\C)\leq (\T \V)(\C)+\de\ ,
\end{equation}
where for $V$ that is defined only on $\SC$, if the next state $\C'\not\in\SC$, then we use $\widehat{\C'}$ as the next state. For the case of zero initiation of the value function, since $V^{0}(\widehat{\C})=0,\ \forall\C\in\Delta_{M}$, this modification does not change the values of $\T^{t}\V(\C)$.

So, Equation \eqref{thm:greedy:first_d} holds for $t=1$. Let's assume that Equation \eqref{thm:greedy:first_d} holds for $t-1$.
Then, since
\begin{equation}
\T(\V+\epsilon)(\C)\leq(\T\V)(\C)+\frac{\gamma}{B}\epsilon,
\end{equation}
and by the monotonicity of the original DP operator (namely, $\T$) it is obtained that
\begin{equation*}
\begin{aligned}
(\gdT\gdT^{t-1} \V)(\C)&\leq (\T\gdT^{t-1} \V)(\C)+\de
\leq\left(\T\left((\T^{t-1}\V)+\de\sum_{i=0}^{t-2}\left(\frac{\gamma}{B}\right)^{i}\right)\right)(\C)+\de
\\&=(\T^{t}\V)(\C)+\de\sum_{i=0}^{t-1}\left(\frac{\gamma}{B}\right)^{i}\ .
\end{aligned}
\end{equation*}
So, Equation \eqref{thm:greedy:first_d} holds for $t$, and therefore, by the induction, Equation \eqref{thm:greedy:first_d} holds for every $t\geq1$.

Now we prove Equation \eqref{eq:ind:assum_d} by induction. We note that
by the fact that $\gdQ(\emptyset,\C)=0$ and by Lemmas \ref{lem:new:2:submodularity:d}, \ref{lem:new:2:monotonicity:d}, which are provided and proved in Section \ref{thm:g_d:supp} in supplementary material and Lemma \ref{lem:theta_bound}, which is provided and proved in Section \ref{thm:g:supp}  in the supplementary material, it is obtained that Lemma \ref{Nemhauser} can be applied on the operator $\gdT$, with $\overline{\theta}_{d}(\C)$ as defined in Equation \eqref{THM:Theta_d} (note that both the term which relates to the almost submodularity and the term which relates to the almost monotonicity are considered in $\overline{\theta}_{d}(\C)$). So, by Lemma \ref{Nemhauser} and by Lemma \ref{def:def:desc} we have
\begin{equation}\label{eq:thm:g_d}
(\T \V)(\C)\leq \frac{1}{\TBb}(\cdT \V)(\C)+(k-1)\overline{\theta}_{d}(\C)
\leq\frac{1}{\TBb}(\gdT \V)(\C)+\frac{\de}{\TBb}+(k-1)\overline{\theta}_{d}(\C)\ .
\end{equation}
In addition, we note that
\begin{equation}\label{for:ind:step1_d}
(\T\beta \V)(\C)=(\T_{\beta\gamma} \V)(\C),
\end{equation}
and that
\begin{equation}\label{for:ind:step2_d}
\TBb\V(\C)\leq\V(\C).
\end{equation}
So, by Equations \eqref{eq:thm:g_d}, \eqref{for:ind:step1_d} and \eqref{for:ind:step2_d} we have that
\begin{equation}\label{eq:ind:assum:0_d}
\TBb\left((\T_{\TBb\gamma} \V)(\C) -\frac{\de}{\TBb}-(k-1)\overline{\theta}_{d}(\C)\right)\leq (\gdT\V)(\C).
\end{equation}
So, by Equation \eqref{eq:ind:assum:0_d}, we can easily see that Equation \eqref{eq:ind:assum_d} satisfies for $t=1$.
Now, Let's assume that Equation \eqref{eq:ind:assum_d} satisfies for $t-1$.

By the fact that 
\begin{equation*}
\begin{aligned}
(\T_{\beta\gamma} \V)(\C) -\beta\gamma\rho(\C) v(\C)
&\leq\sum_{m\in\M}\sum_{\ell\in\Ls}\C(m)P(\ell|m,w')\left(1+\beta\gamma\left(\V(\C'_{l,\w,\C})-v(\C)\right)\right)
\\&\leq (\T_{\beta\gamma} (\V -v(\C)))(\C)\ ,
\end{aligned}
\end{equation*}
where $\rho(\C)$ is defined in Theorem \ref{thm:g_d}, $v(\cdot)$ is a function of $\C\in\SC$ and $w'$ is the chosen action by the DP operator in $(\T_{\beta\gamma} \V)(\C)$ and by Equation \eqref{for:ind:step1_d} it is obtained that
\begin{equation}\label{eq:ind:assum:b:1_d}
\begin{aligned}
&(\T^{t}_{\TBb\gamma} \V)(\C) -\sum_{i=1}^{t-1}\left(\TBb\gamma\rho(\C)\right)^{i}\left(\frac{\de}{\TBb}+(k-1)\overline{\theta}_{d}(\C)\right)
\\&\leq
\left(\T\TBb\left(\T^{t-1}_{\TBb\gamma} \V -\sum_{i=0}^{t-2}\left(\TBb\gamma\rho(\C)\right)^{i}\left(\frac{\de}{\TBb}+(k-1)\overline{\theta}_{d}(\C)\right)\right)\right)(\C)
\triangleq\Upsilon(\C)\ .
\end{aligned}
\end{equation}
Furthermore, since we assume that Equation \eqref{eq:ind:assum_d} satisfies for $t-1$ and by the monotonicity of the operator $\T$, we have
\begin{equation}\label{eq:ind:assum:b:1_5_d}
\begin{aligned}
\Upsilon(\C)\leq(\T\gdT^{t-1}\V)(\C)
\end{aligned}.
\end{equation}
Then, by Equation \eqref{eq:thm:g_d} we have
\begin{equation}\label{eq:ind:assum:b:2_d}
(\T \gdT^{t-1}\V)(\C)
\leq\frac{1}{\TBb}(\gdT \gdT^{t-1}\V)(\C)+\frac{\de}{\TBb}+(k-1)\overline{\theta}_{d}(\C)\ .
\end{equation}
So, by Equations \eqref{eq:ind:assum:b:1_d} \eqref{eq:ind:assum:b:1_5_d}, and \eqref{eq:ind:assum:b:2_d} it is obtained that
\begin{equation*}
\TBb\left((\T^{t}_{\TBb\gamma} \V)(\C) -\sum_{i=0}^{t-1}\left(\TBb\gamma\rho(\C)\right)^{i}\left(\frac{\de}{\TBb}+(k-1)\overline{\theta}(\C)\right)\right)
\leq(\gdT^{t}\V)(\C)\ .
\end{equation*}
So, it is obtained that Equation \eqref{eq:ind:assum} satisfies also for $t$. Therefore, by induction, Equation \eqref{eq:ind:assum} satisfies for every $t\geq 1$
\qed

\section{Additional Experiments}
\label{sup:add:exp}

For the additional experiments, we considered the case of $\M=4$, $|\Ls|=21$, $\K=3$. The scores were chosen as follows: For all types, the termination score was $p_m=0.5$. Four items were chosen i.i.d.\ uniformly at random from the interval $[0, 0.6]$. The remaining $16$ items where chosen such that for each user type, $4$ items are uniformly distributed in $[0.5,1]$ (strongly related to this type), while the other $6$ are drawn uniformly from $[0,0.5]$. 
We repeated the experiment $50$ times for $\gamma=0.7$ and $130$ times for $\gamma=1$, where for each repetition a different set of scores was generated and $100,000$ sessions were generated (a total of $5M$ and $13M$ sessions respectively).

In Figure \ref{figure:experiment0_7} we present the average session length under the \emph{optimal, greedy} and \emph{simple greedy} CPs for different numbers of iterations executed for computing the Value function for $\gamma=0.7$. The average length that was achieved by the \emph{random} CP is $1.4374$, much lower than that of the other methods. The standard deviation is smaller that $2\times10^{-3}$ in all of our measures. As shown in Figure \ref{figure:experiment0_7}, the extra comparison step in the \emph{greedy} CP compared to the \emph{simple greedy} CP substantially improves the performance.

\begin{figure}[ht]
	\centering
	\includegraphics[width=0.45\textwidth]{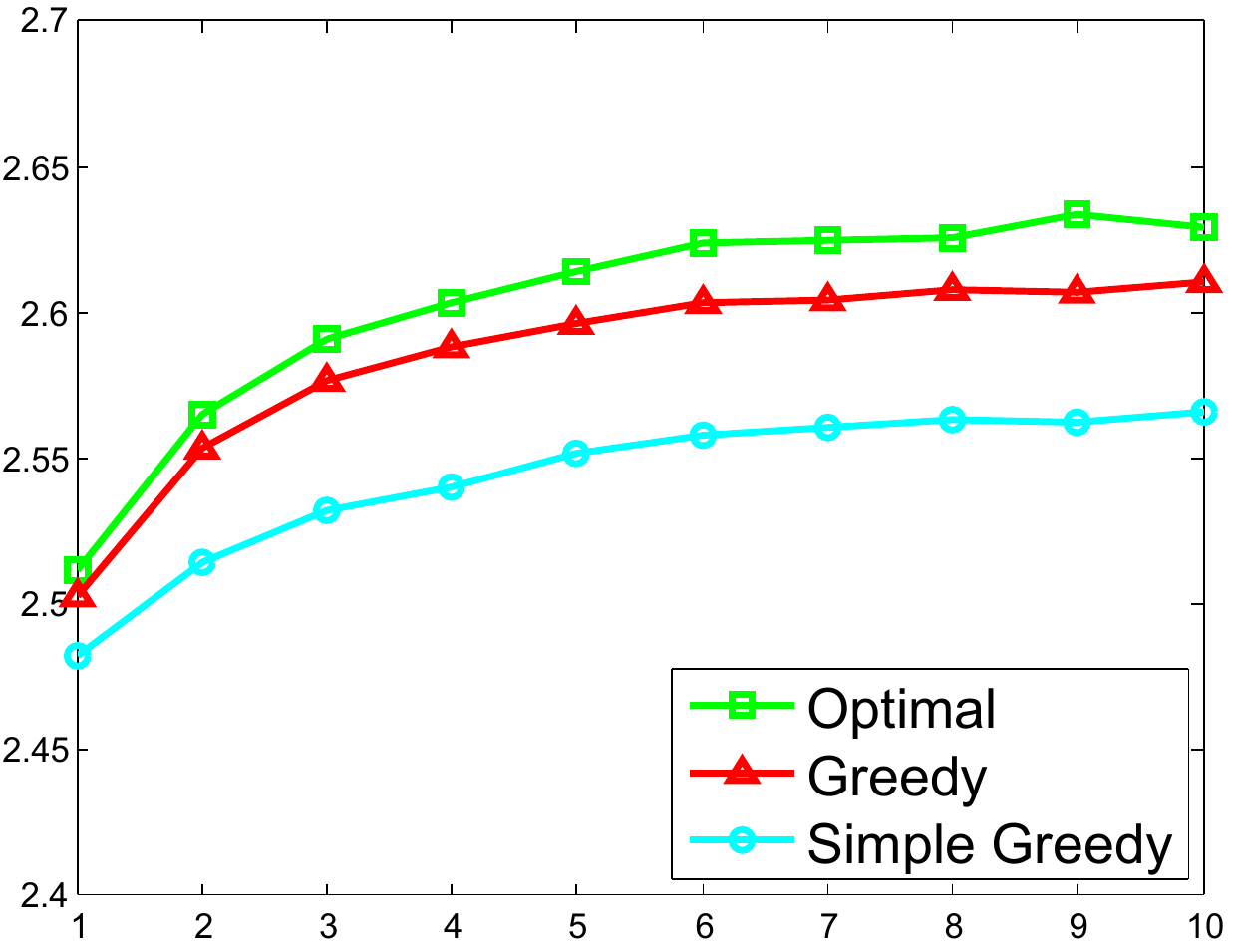}
	\caption{$\gamma=0.7$. Average session length under the \emph{optimal, greedy} and \emph{simple greedy} ($y$-axis) CPs vs. number of iterations of the related VI computation ($x$-axis). The average length of the \emph{random} CP is $1.4374$ (not shown).}\label{figure:experiment0_7}
\end{figure}

In Figure \ref{figure:experiment1} we present the average session length under the \emph{optimal, greedy} and \emph{simple greedy} CPs for different numbers of iterations executed for computing the Value function for $\gamma=1$. The average length that was achieved by the \emph{random} CP is $1.4499$, much lower than that of the other methods. The standard deviation is smaller that $1.5\times10^{-3}$ in all of our measures. As shown in Figure \ref{figure:experiment0_7}, the extra comparison step in the \emph{greedy} CP compared to the \emph{simple greedy} CP substantially improves the performance. 

\begin{figure}[ht]
	\centering
	\includegraphics[width=0.45\textwidth]{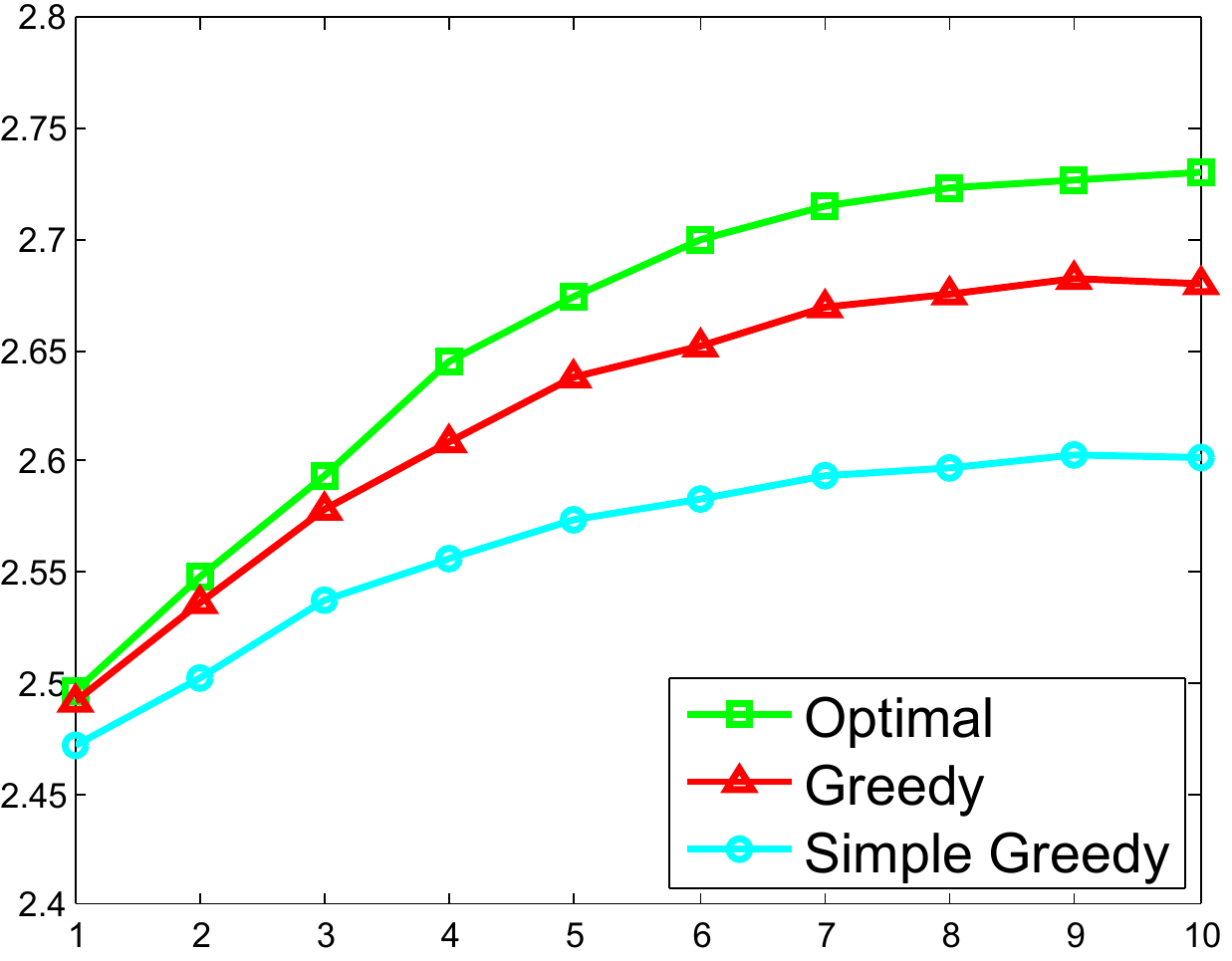}
	\caption{$\gamma=1$. Average session length under the \emph{optimal, greedy} and \emph{simple greedy} ($y$-axis) CPs vs. number of iterations of the related VI computation ($x$-axis). The average length of the \emph{random} CP is $1.4499$ (not shown).}\label{figure:experiment1}
\end{figure}

\section{Example for non-Monotone and non-Submodular $Q$ Function}
\label{app:Q_example}

In this section we provide an example for a reward function that is monotone and submodular with a corresponding $Q$ function that does not share these properties. We define $S=\{1,2,3\}$ as the state space,  $L=\{1,2,3\}$ as the basis to the action space $A=L \times L \cup L \cup \{ \emptyset \}$. The reward function is defined as $r(a=\{i,j\},s) = s \cdot (i+j)$, $r(a=\{i\},s) = s \cdot i$, $r(\emptyset, s)=0$. The transition function is deterministic with $p(s \ | \ a=\{i,j\},s')=1$ for $s \neq i,j$, $p(s \ | \ a=\{i\},s') = 1$ for $s=i$ and $p(s \ | \ \emptyset,s') = 1$ for $s=s'$. The reward function $r$, when viewed as a function of the action is linear and clearly monotone submodular.
\begin{enumerate}
\item
One can verify that for $\gamma=0.5$ and the zero initialization of the value function, in the third applying of the G-VI operator the Q function is not monotone as $Q(a=\{3\},s=1)=11.75$ and $Q(a=\{3,1\},s=1)=10.75$. Also, $Q$ is not submodular as $Q(a=\{2\},s=1)-Q(a=\emptyset,s=1)=3.5$ and $Q(a=\{1,2\},s=1)-Q(a=\{1\},s=1)=5.5$.
\item
Also, for the same $\gamma$ and the optimal value function, the $Q$ function is not monotone as $Q(a=\{3\},s=1)=14$ and $Q(a=\{3,1\},s=1)=12$. Also, $Q$ is not submodular as $Q(a=\{2\},s=1)-Q(a=\emptyset,s=1)=3$ and $Q(a=\{1,2\},s=1)-Q(a=\{1\},s=1)=6$.
\end{enumerate}

\end{document}